\documentclass[preprint, 11pt]{article}

\usepackage{times}
\usepackage{caption,fullpage}
\usepackage{todonotes}
\usepackage{algorithmic}
\usepackage{algorithm}
\usepackage{amsmath}
\usepackage{tikz}
\usepackage{natbib}

\usepackage{titletoc}
\usepackage[page, header, toc, page]{appendix}

\usepackage{xcolor}
\usepackage{colortbl}
\usetikzlibrary{calc,intersections,through,arrows,angles, quotes}
\usepackage{tkz-euclide}

\usepackage[utf8]{inputenc} 
\usepackage[T1]{fontenc}    

\usepackage[colorlinks,linkcolor=red,anchorcolor=blue,citecolor=blue]{hyperref}
\definecolor{light-gray}{gray}{0.9}
\definecolor{darkgreen}{rgb}{0,0.5,0}
\definecolor{darkblue}{rgb}{0.0,0.0,0.65}
\definecolor{darkred}{rgb}{0.55,0.0,0.0}
\definecolor{cadetblue}{rgb}{0.37, 0.62, 0.63}
\definecolor{cadetgrey}{rgb}{0.57, 0.64, 0.69}
\hypersetup{
colorlinks = true,
citecolor  = darkblue,
linkcolor  = darkred,
filecolor  = darkblue,
urlcolor   = darkblue,
}

\usepackage[normalem]{ulem} 
\usepackage{url}         \usepackage{booktabs}    \usepackage{amsfonts}    \usepackage{nicefrac}    \usepackage{microtype}   \usepackage{amssymb}

\usepackage{url}            %
\usepackage{nicefrac}      
\usepackage{amsfonts, amsmath, amssymb, amsthm, bm, mathtools,dsfont}
\mathtoolsset{showonlyrefs}
\usepackage{mathrsfs}  
\usepackage{algorithm} 

\usepackage{enumitem}

\setlist{leftmargin=0.3in} 

\usepackage{thmtools} 
\usepackage{thm-restate}
\usepackage{float,caption}
\usepackage{subcaption}
\usepackage{wrapfig}
\usepackage{nicefrac}
\usepackage{xfrac}
\usepackage{booktabs}
\usepackage{makecell}
\usepackage{hhline}
\usepackage{longtable}
\usepackage{rotating}
\usepackage{array}
\usepackage{multicol}
\usepackage{multirow}

\usepackage{mdframed}
\mdfsetup{
linecolor=black!40,
linewidth=0.5pt,
innerleftmargin = 2pt,
innertopmargin = 4pt,
innerrightmargin= 2pt,
innerbottommargin = 4pt,
}

\newtheorem{theorem}{Theorem} 
\newtheorem{definition}{Definition}
\newtheorem{lemma}{Lemma}
\newtheorem{corollary}[lemma]{Corollary}
\newtheorem{proposition}{Proposition}

\newtheorem{assumption}{Assumption}

\declaretheorem[shaded={rulecolor=black, rulewidth=0.5pt, bgcolor=gray!7}, name=Theorem, sibling=theorem]{thmbox}

\def\bz{{\mathbf z}}

\def\bx{{\mathbf x}}

\def\bw{{\mathbf w}}

\def\bv{{\mathbf v}}

\newcommand{\tS}{\widehat{S}}

\newcommand{\tn}{\tau} 
\newcommand{\cx}{C_{\bx}}

\newcommand{\by}{\mathbf{y}}
\newcommand{\eby}{\mathbf{Y}}

\newcommand{\ebx}{\mathbf{X}}
\newcommand{\bby}{\overline{\mathbf{y}}}
\newcommand{\bbx}{\overline{\mathbf{x}}}

\newcommand{\bg}{\mathbf{g}}
\newcommand{\bu}{\mathbf{u}}

\newcommand{\Unif}{\mathrm{Unif}}
 
\newcommand{\eps}{\varepsilon}

\newcommand{\OO}[1]{\mathcal{O}\left(#1\right)}
\newcommand{\TT}[1]{{\Theta}\left(#1\right)}
\newcommand{\anchor}{\mathbf{a}}

\newcommand{\update}{{\boldsymbol \delta}}

\newcommand{\betas}{\beta_\star}
\newcommand{\Ds}{D_\star}
\newcommand{\mus}{\mu_\star}
\newcommand{\etas}{\eta_\star}

\newcommand{\bt}{\zeta}
\newcommand{\bts}{\zeta_\star}

\newcommand{\algcolor}[1]{\textcolor{blue!70!black}{#1}}
\newcommand{\algcomment}[1]{\textcolor{blue!70!black}{\small{\texttt{\textbf{//\hspace{2pt}#1}}}}}

\DeclareMathOperator*{\argmin}{arg\,min}

\newcommand{\inp}[2]{\left\langle #1,#2 \right\rangle}
\newcommand{\R}{\mathbb{R}}
\newcommand*{\E}{\mathbb{E}}

\newcommand{\regret}{{\sf Regret}}
\newcommand{\norm}[1]{\left\|#1\right\|}

\newcommand{\dregret}{{\sf Regret}^{[\beta
]}}
 
\newcommand{\subopt}{\Delta_F}

\newcommand{\regnorm}[2]{\left\| #2\right\|^{[#1]}}

\newcommand{\stograd}{\textsc{StoGrad}}

\newcommand{\D}{\mathrm{d}}

\title{General framework for online-to-nonconvex conversion:\\
Schedule-free SGD is also effective for nonconvex optimization} 

\author{Kwangjun Ahn$^\star$ \\
Microsoft Research 
\\ {\small \texttt{kwangjunahn@microsoft.com}}\and Gagik Magakyan$^\star$
\\
MIT 
\\ {\small \texttt{gagmag@mit.edu}} 
\and Ashok Cutkosky\\ 
Boston University
\\ {\small \texttt{ashok@cutkosky.com}}
}

\begin{document}

\maketitle
\begingroup\renewcommand\thefootnote{$\star$}
\footnotetext{Equal contribution.}
\endgroup

\begin{abstract}  
This work investigates the effectiveness of schedule-free methods, developed by A. Defazio et al. (NeurIPS 2024), in nonconvex optimization settings, inspired by their remarkable empirical success in training neural networks. Specifically, we show that schedule-free SGD achieves optimal iteration complexity for nonsmooth, nonconvex optimization problems. Our proof begins with the development of a general framework for online-to-nonconvex conversion, which converts a given online learning algorithm into an optimization algorithm for nonconvex losses. 
Our general framework not only recovers existing  conversions but also leads to two novel conversion schemes. Notably, one of these new conversions corresponds directly to schedule-free SGD, allowing us to establish its optimality. 
Additionally, our analysis provides valuable insights into the parameter choices for schedule-free SGD, addressing a theoretical gap that the convex theory cannot explain.

\end{abstract}


\section{Introduction}
\label{sec:intro}

Training large-scale neural network models, such as large language models, requires a well-designed optimization strategy to ensure stable and fast convergence. 
For instance, training typically requires a carefully designed optimizer, such as the Adam optimizer \citep{kingma2014adam}, along with meticulously tuned learning rate scheduling.

Recently, \citet{defazio2024road} introduced the schedule-free method, which achieves impressive training performance without any need for learning rate scheduling.
 In brief, the schedule-free method is an add-on scheme that can be applied to any chosen base optimizer, converting it into a schedule-free variant. While this method has shown strong empirical performance in training large neural network models, its theoretical analysis has, to date, been limited to the convex setting \citep{defazio2024road}.
 Our aim is to extend the theoretical understanding of schedule-free methods to nonconvex optimization.

 In this work, as an initial step, we focus on the version where the base optimizer is chosen as SGD, referred to as \emph{schedule-free SGD}. For a given learning rate $\gamma > 0$ and interpolation weights $c_t, \kappa_t \in [0, 1]$, the updates of schedule-free SGD maintain three sequences of iterates, $\bx_t$, $\by_t$, and $\bz_t$, as follows:
\begin{align}
\tag{\textsc{SF-SGD}}\label{schefree}
\begin{cases}
&\bx_{t} =  (1-c_{t})  \bx_{t-1} + c_{t}  \bz_{t}, \\ 
&\by_t = (1-\kappa_t) \bz_t + \kappa_t \bx_t, \\
&\bg_t = \text{a stochastic gradient at}~\by_t, \\ 
&\bz_{t+1}= \bz_t -\gamma \bg_t.
\end{cases}
\end{align} 
Here, $\bz_t$ corresponds to the base SGD trajectory, $\bx_t$ maintains a (weighted) average of $\bz_t$, and $\by_t$ is an interpolation between $\bx_t$ and $\bz_t$ where the stochastic gradient is computed.

\subsection{Our main result and approach}
\label{sec:approach}

To understand the effectiveness of schedule-free SGD for training neural networks, we analyze the method in the \textbf{nonsmooth and nonconvex} setting~\citep{zhang2020complexity,davis2022gradient,tian2022finite}.  Specifically, we adopt the $(\lambda,\epsilon)$-stationarity criterion from \citep{zhang2024random,ahn2024adam} (\autoref{def:reg_station}), which seeks approximate Goldstein stationary points \citep{goldstein1977optimization}.
This criterion emerges as a practical framework for analyzing optimization algorithms, particularly considering recent findings that show practical optimizers like SGD with momentum and Adam are theoretically optimal  \citep{zhang2024random,ahn2024adam}.

Our main results demonstrate that schedule-free SGD is optimal not only for convex optimization, as established by \citet{defazio2024road}, but also for the most challenging case of nonsmooth and nonconvex optimization. Our main results can be summarized informally as follows:

\begin{thmbox}[Informal; see \autoref{sec:SF}] \label{thm:informal}
Schedule-free SGD \eqref{schefree}, with an appropriate choice of parameters $\gamma, c_t, \kappa_t$, achieves optimal rates for nonsmooth and nonconvex $F$.
\end{thmbox}  

Our proof technique leverages the online-to-nonconvex conversion framework pioneered by \citet{cutkosky2023optimal}. In essence, this framework takes an online learner as input and outputs an optimization algorithm. As its name suggests, this framework translates online learning guarantees into nonconvex guarantees, analogous to the well-known online-to-batch conversion for convex settings \citep{cesa2004generalization}.
As suggested by our title, we first introduce a general framework for online-to-nonconvex conversion. This framework not only encompasses the previous conversion \citep{cutkosky2023optimal,zhang2024random,ahn2024adam} as a special case but also enables new conversion schemes, as we present in \autoref{sec:anchor} and \autoref{sec:SF}.

With the new conversion schemes, our main observation is that one of these novel conversions, outlined in \autoref{alg:SF}, directly corresponds to schedule-free SGD. Specifically, by choosing a basic online mirror descent as the online learner in \autoref{alg:SF}, we naturally recover the schedule-free SGD algorithm.
From this, our main result, \autoref{thm:informal}, follows.

Notably, our approach provides fresh practical insights into the parameter choice of schedule-free methods, which cannot be explained by prior convex analysis in \citet{defazio2024road}.  First, our results suggest that setting $\kappa_t$ close to $1$ is advantageous for nonconvex optimization. This finding clarifies the curious importance of choosing $\kappa_t$ near $1$ in \citep{defazio2024road} for empirical performance—a phenomenon not previously explained by convex theory.
Moreover, our findings also supports choosing the base optimizer step size $\gamma$ much larger than usual values, supporting the practical choices made in \citep{defazio2024road}.

\subsection{Related work}
\label{sec:related}

Our work builds on a line of research focused on convergence guarantees for nonsmooth, nonconvex optimization. Intuitively, our convergence criterion aims to find approximate Goldstein stationary points \citep{goldstein1977optimization}. The formal study of iteration complexity for finding approximate Goldstein stationary points was initiated by \citet{zhang2020complexity} and has since garnered significant interest \citep{davis2022gradient,tian2022no, lin2022gradient,chen2023faster,jordan2023deterministic,cutkosky2023optimal,kornowski2024algorithm}. Alternative convergence notions are also widely used, including approaches based on the Moreau envelope or by imposing weak convexity conditions \citep{davis2018subgradient, davis2022escaping}.
In particular, the criterion we consider in this work, formalized in \autoref{def:reg_station}, follows the relaxed version proposed by \citet{zhang2024random}, which slightly modifies the original criterion from \citet{zhang2020complexity}.

 As discussed in \autoref{sec:approach}, our approach specifically builds on leverages the online-to-nonconvex conversion framework introduced by \citet{cutkosky2023optimal}, a versatile framework that established the first optimal algorithm for stochastic nonsmooth, nonconvex optimization. 
Similar to the well-known online-to-batch conversion \citep{cesa2004generalization}, which transforms an online learning algorithm into an optimization algorithm for convex losses, this framework adapts this transformation to nonconvex settings.
A key distinction, as noted by \citet{ahn2024understanding}, lies in how each framework sets the iterates of the resulting optimization algorithm. 
While the online-to-batch conversion sets the iterates, $\bx_t$, directly according to the online learner, the online-to-nonconvex conversion framework determines the increments, $\bx_t - \bx_{t-1}$, based on updates from the online learner.

  Practically, the online-to-nonconvex conversion framework offers a promising perspective on the success of popular optimizers. By choosing standard online learners within this framework, we recover practical algorithms: online mirror descent corresponds to SGD with momentum \citep{zhang2024random}, while a variant of follow-the-regularized-leader (FTRL) aligns with Adam \citep{ahn2024adam}. 
  This natural alignment with widely used optimizers underscores the framework’s potential as a powerful tool for understanding the success of practical optimization algorithms.
  Building on this, our work further demonstrates the framework’s promise by showing that it also captures the highly effective schedule-free SGD algorithm.

Lastly, we note that the existing theory for the schedule-free method is limited to convex losses, analyzed through online-to-batch conversions \citep{defazio2024road}. Specifically, when the base optimizer is selected as an online learner, \citet{defazio2024road} demonstrate that schedule-free methods effectively transform online learning guarantees into optimization guarantees for the last iterate, similar to the results of \citep{cutkosky2019anytime, kavis2019unixgrad}. However, this convex theory does not account for the empirical success of schedule-free methods in training neural networks.  
Our work extends these insights by showing that schedule-free SGD is also effective for nonsmooth, nonconvex optimization.

\section{Preliminaries}
\label{sec:prelim}

We first introduce the key assumptions and the notion of convergence.
Throughout this paper, unless specified otherwise, $\norm{\cdot}$ denotes the $L_2$ norm. 

\subsection{Assumptions on the loss function and stochastic gradients}

Following \citet{cutkosky2023optimal}, we consider optimizing a loss function $F$ that satisfies the following conditions.

\begin{assumption} \label{assump}
Let $F:\R^d \to \R$ be a differentiable function with the following properties:
\begin{itemize}

\item Let $\subopt \coloneqq F(\bx_0) -\inf_{\bx} F(\bx)$.

\item For any two points $\bx$ and $\bw$, $F(\bx)-F(\bw) = \int_0^1 \inp{\nabla F(\bw+t(\bx-\bw))}{\bx-\bw}\D t$. 

\item  $F$ is $G$-Lipshitz, \emph{i.e.}, for any point $\bx$, $\norm{\nabla F(\bx)}\leq G$.
 
\end{itemize}

\end{assumption}
Here, the second condition, called \emph{well-behavedess} in \citep[Definition 1]{cutkosky2023optimal}, is a mild regularity condition. 
For any locally Lipschitz function $F$, applying an arbitrarily small perturbation to the function is sufficient to ensure this condition \citep[Proposition 2]{cutkosky2023optimal}. 

We optimize the loss function $F$ by accessing information  through a stochastic gradient oracle, which is formalized as follows:

\begin{assumption}[{\bf Stochastic gradient oracle}] 
We assume access to a stochastic gradient oracle at any point $\bx$. More formally, for any given point $\bx$, each call to the oracle independently returns a stochastic gradient $\bg$ that satisfies the following properties:
\[
\E[\bg] = \nabla F(\bx), \quad \text{and} \quad \E\left[\norm{\bg - \nabla F(\bx)}^2\right] \leq \sigma^2.
\]
We denote the stochastic gradient oracle as $\stograd$. When the oracle returns the stochastic gradient $\bg$ at point $\bx$, we write this as $\bg \gets \stograd(\bx)$.
\end{assumption}

\subsection{Approximate  Goldstein stationary point}

For the notion of optimality, we follow \citep{zhang2024random,ahn2024adam} and consider the following notion of stationarity for  nonconvex and nonsmooth functions.
This notion can be regarded as an approximate version  of the notion of a Goldstein stationarity point~\citep{goldstein1977optimization}.

\begin{definition}[{\bf $(\lambda,\eps)$-stationary point}]
\label{def:reg_station}
Suppose $F:\R^d\to \R$ is differentiable.
We say $\bx$ is a $(\lambda,\eps)$-stationary point of $F$ if $\regnorm{\lambda}{\nabla F(\bx)}  \leq \eps$, where 
\begin{align}
\regnorm{\lambda}{\nabla F(\bx)} \coloneqq \inf_{\substack{p \in \mathcal{P}(\R^d),\\\E_{\bw\sim p}[\bw] =\bx} } \left\{\norm{\E[\nabla F(\bw)]} + \lambda \cdot \E\norm{\bw-\bx}^2 \right\}\,.
\end{align}
\end{definition}

Our algorithms will identify $(\lambda,\eps)$-stationary points using $\OO{\lambda^{1/2}\epsilon^{-7/2}}$ calls to a stochastic gradient oracle, which is the optimal rate \citep{zhang2024random}.

To further motivate this definition, we remark that $(\lambda,\eps)$-stationary points retain the desirable properties of Goldstein stationary points. Specifically, the following result {\citep[Lemma 2.3]{zhang2024random}} demonstrates that, akin to Goldstein stationary points, $(\lambda,\eps)$-stationary points can be reduced to first-order stationary points with appropriate choices of $\lambda$ when the objective function is smooth or second-order smooth.

\begin{proposition}  \label{lem:convert}
If $F$ is $L$-smooth, then an $(L^2 \eps^{-1},\eps)$-stationary point $\bx$ of $F$ satisfies $\norm{\nabla F(\bx)}\leq 2\eps$. Moreover, if $F$ is $H$-second-order-smooth, then an $(H/2,\eps)$-stationary point $\bx$ of $F$ satisfies $\norm{\nabla F(\bx)}\leq 2\eps$. 
\end{proposition}


Note that another popular notion of approximate Goldstein stationarity is called $(\delta,\eps)$-stationarity due to  \citet{zhang2020complexity}.
However, the algorithms that achieve optimal complexity under that notion often require clipping on the momentum term \citep{cutkosky2023optimal}, which introduces deviations from the practical optimization algorithms. Hence, in this work, we adopt the notion of $(\lambda,\eps)$-stationarity, for which it has been demonstrated that optimal algorithms does not require clipping operations \citep{zhang2024random}.
We also remark that algorithms that identify $(\lambda, \eps)$-stationary points can also identify $(\delta,\eps)$-stationary points when $F$ is Lipschitz, as demonstrated in {\citep[Lemma 2.4]{zhang2024random}}.
 
As mentioned above, the stochastic gradient oracle complexity of finding a $(\lambda,\eps)$ stationary point is $\Theta(\lambda^{1/2}\eps^{-7/2})$ \citep{zhang2024random}. By \autoref{lem:convert}, any algorithms achieving this rate (such as the ones we will present), can also find a point $\bx$ with $\|\nabla F(\bx)\|\le \eps$ in $O(\epsilon^{-4})$ oracle evaluations when $F$ is smooth  and $O(\epsilon^{-3.5})$ oracle evaluations when $F$ is second-order  smooth. These are the optimal rates for their respective function classes \citep{arjevani2023lower, arjevani2020second}.

\subsection{Online learning}

In this section, we provide a brief background on online learning. Online learning is modeled as a sequential decision-making process over $T$ rounds.
In each round $t$, the online learner chooses a point $\update_t\in\R^d$, and then a loss function $\ell_t: \R^d \to \R$ is revealed.
The learner then incurs a loss $\ell_t(\update_t)$. 
The choice of $\update_t$ is based on the previous loss sequence $\ell_{1:t-1}$, and after selecting  $\update_t$, the learner observes the next loss $\ell_t$.

The performance of the online learner is measured using the \emph{regret} with respect to a comparator $\bu$, formally defined as: 
\begin{align}
\regret_{T}(\bu) \coloneqq    \sum_{t=1}^{T} (\ell_t(\update_t)- \ell_t(\bu)).
\end{align}  
However, recent works \citep{cutkosky2023optimal, ahn2024understanding} demonstrate that when designing nonconvex optimization algorithms, base online learners must be adaptive to time-varying comparators, also known as the \emph{dynamic} regret setting.
Inpsired by \citep{ahn2024understanding,zhang2024random,ahn2024understanding}, we  design such dynamic online learners by considering the following   ``discounted'' version of  regret for online learners.
\begin{definition}[{\bf Discounted regret}]
Consider the loss sequence $\ell_{1:T}$.
For any $T\geq 1$, we define the discounted regret of an online learner with respect to a comparator $\bu$ as:
\begin{align}
\dregret_T(\bu) \coloneqq  \sum_{t=1}^T \beta^{T-t}\left( \ell_t(\update_t)- \ell_t(\bu) \right).
\end{align} 
\end{definition}
For instance, \citet{ahn2024understanding} demonstrate that online learners with low discounted regret can achieve low dynamic regret through what they call the \emph{discounted-to-dynamic} conversion.
Additionally, discounted regret has been shown to be a more effective metric for designing adaptive online learners in dynamic environments across various contexts, including conformal prediction \citep{zhang2024discounted} and online linear regression~\citep{jacobsen2024online}.

\section{General framework for online-to-nonconvex conversion}

In this section, we introduce a general scheme (presented in \autoref{alg:general}) for converting online learning guarantees into  nonconvex optimization guarantees.
As a preview, we will demonstrate that our \autoref{alg:general} not only recovers existing approaches \citep{cutkosky2023optimal, zhang2024random, ahn2024adam} as special cases, but also enables the design of novel conversion methods.
Let us begin with the algorithm procedure.

\begin{algorithm}[H]
\caption{General scheme for online-to-nonconvex conversion}
\label{alg:general}
\begin{algorithmic}[1]
\STATE \textbf{Input:} Initial iterates $\bx_0=\bw_0$, an online learner $\mathcal{A}$, and $T\in \mathbb{N}$, regularization strength  $\mu \geq 0$.
\FOR{$t = 1,2, \ldots, T$}
\STATE Receive $\update_{t}$ from $\mathcal{A}$.
\STATE   \algcolor{Choose $\bx_t$ arbitrarily.}    \algcomment{The design choice}
\STATE  Update $\bw_{t} = \bx_{t} +  \update_{t}$. 
\STATE Update $\by_{t} =   \bx_{t} +  s_t\update_{t}$ where $s_{t}$ is drawn uniformly from $[0,1]$ \emph{i.i.d}. 
\STATE Compute $\bg_{t} \leftarrow \stograd(\by_{t})$.
\STATE Send loss $\ell_{t}(\cdot)
=  \langle  \bg_{t}, \cdot \rangle + \frac{\mu}{2}\norm{\cdot}^2$ to $\mathcal{A}$.
\ENDFOR   
\end{algorithmic}
\end{algorithm}

Overall, \autoref{alg:general} generates three sequences of iterates: $\bx_t$, $\by_t$, and $\bw_t$. At each iteration, the stochastic gradients observed so far are fed into the online learner, which outputs $\update_t$.

The main design feature of \autoref{alg:general} that gives it flexibility is the ability to choose $\bx_t$ arbitrarily at each iteration. However, there are technical properties we want $\bx_t$ to satisfy in order to achieve better nonconvex guarantees, which are detailed in \autoref{sec:user}.

The online learner's output $\update_t$ is then used alongside $\bx_t$ to compute $\bw_t$ according to the update rule $\bw_t = \bx_t + \update_t$. Additionally, the iterates $\by_t$ are sampled uniformly from the line segment connecting $\bx_t$ and $\bw_t$, and the stochastic gradients are computed at these $\by_t$ iterates.

\subsection{Output of the general scheme}

To establish nonconvex optimization guarantees, we utilize the exponential moving average (EMA) of the $\by_t$ iterates, similar to \citep{zhang2024random,ahn2024adam}. We begin by formally defining these EMA iterates.

\begin{definition}[{Exponential moving average (EMA) of iterates}] \label{def:ema}
Given a discount factor $\beta \in (0,1)$ and a sequence of iterates $\{\by_s\}_{s=1}^t$, the $\beta$-EMA of the sequence up to time $t$, denoted as $\bby_t$, is defined as
\begin{align}
\bby_t = \frac{1-\beta}{1-\beta^t} \sum_{s=1}^t \beta^{t-s} \by_s.
\end{align}  
\end{definition}

In particular, as we will see in \autoref{lem:master}, the final output will be the random EMA iterate $\bby_\tau$, where $\tau$ is a carefully selected random index, defined as follows.

\begin{definition}[{Random index distribution}] \label{def:random}
Let $\tn$ be a random index distributed over $\{1, 2, \dots, T\}$ with the following distribution:
\begin{align}
\Pr(\tn = t) = \begin{cases}
\frac{1 - \beta^t}{T}, & \text{for}~~ t = 1, \dots, T-1, \\
\frac{1}{1-\beta} \cdot \frac{1 - \beta^T}{T}, & \text{for}~~ t = T.
\end{cases}
\end{align}
\end{definition}

We note that previous works \citep{zhang2024random, ahn2024adam} select the output uniformly at random from the sequence $\{\bby_t\}_{t=1}^T$. Our carefully designed random EMA iterate from \autoref{def:random} leads to an improved conversion result, offering stronger nonconvex optimization guarantees, as we will discuss in \autoref{sec:warmup}.

Notice that as $\beta$ approaches $1$, the random index $\tn$  assigns a significantly higher probability (by a multiplicative factor of $\frac{1}{1-\beta}$) to the final index $T$.  This aligns with common practice, where the final iterate—rather than an average of previous iterates—is often used as the output.
Indeed, we will select $\beta$ very close to $1$ for our nonconvex guarantees (we will choose $\beta=1-O(\eps^2)$).

\subsection{Conversion guarantees}
\label{sec:conversion}

Given the description of the algorithm and its output, we now present the online-to-nonconvex conversion guarantees. We begin by introducing a key definition that quantifies the stability of the $\bx_t$ sequence.

\begin{definition}[$\bx$-iterate stability] \label{def:stability}
Consider the iterates generated by \autoref{alg:general}.  The {iterate stability factor}, denoted by $\cx \geq 0$, is the smallest nonnegative constant such that:
\begin{align*} 
\mathbb{E} \sum_{t=1}^{T} \|\mathbf{x}_t - \mathbf{x}_{t-1}\|^2 \leq \cx \cdot \mathbb{E} \sum_{t=1}^{T} \|\update_t\|^2.
\end{align*}   
\end{definition}

With the concept of the iterate stability factor, we can now state the conversion result.

\begin{lemma}[{\bf Generic online-to-nonconvex conversion}] \label{lem:master}
Consider the iterates generated according to \autoref{alg:general}.
For $\beta \in (0, 1)$ and $D > 0$, define  the comparators for online learner as follows:
\begin{align}
\forall t \in[T], \quad \mathbf{u}_t \coloneqq -D \frac{\sum_{s=1}^{t} \beta^{t - s} \nabla F(\by_s)}{\left\|\sum_{s=1}^{t} \beta^{t - s} \nabla F(\by_s)\right\|}  .
\end{align}
Then, as long as the regularization strength satisfies $\mu \geq 8\lambda D \left(1 + { \cx}{(1-\beta)^{-2}} \right)$, the following holds:
\begin{align*}
\E_{\tn }  \regnorm{\lambda}{\nabla F(\bby_{\tn})}  &\leq  \frac{ 1}{DT}\left(\beta \E\left[\dregret_T(\bu_T) \right] +  (1-\beta)\sum_{t=1}^{T}     \E\left[\dregret_t(\bu_t)\right]\right)    \\
&\quad + \frac{1}{DT} \E \sum_{t=1}^{T} (F(\bx_{t}) - F(\bw_{t}))    +  \frac{\mu D}{2}    + \frac{\sigma }{T\sqrt{1-\beta}}  + \sigma \sqrt{1-\beta}.
\end{align*}  
\end{lemma}
\begin{proof}
See \autoref{pf:lem:master}.
\end{proof}

\autoref{lem:master} provides an upper bound on the main quantity of interest, $\E_{\tn} \regnorm{\lambda}{\nabla F(\bby_{\tn})}$. One of the main terms in this upper bound is the sum of the discounted regret terms, which indicates that improved performance by the online learner leads to better guarantees in nonconvex optimization. Thus,  \autoref{lem:master} effectively translates the online learning guarantee into a nonconvex optimization guarantee, as suggested by its name, \emph{online-to-nonconvex} conversion.

However, in its current form, the presence of several additional terms in the upper bound makes the result less interpretable. Before demonstrating the strength of this general conversion framework, we first reformulate \autoref{lem:master} into a more interpretable and user-friendly form.

\subsection{User-friendly nonconvex optimization guarantees}
\label{sec:user}

In this section, we apply a concrete discounted regret bound to \autoref{lem:master} to derive a more user-friendly nonconvex guarantee, which will be used throughout the remainder of the paper.
In particular,   a discounted version of composite objective online mirror descent (OMD) \citep{beck2003mirror,duchi2010composite,zhang2024random} achieves the following discounted regret bound.

\begin{lemma} \label{lem:omd-regret}
Let $\beta \in (0,1)$, $\mu \geq 0$, $\eta > 0$, and a sequence of vectors $\{\bg_t\}_{t=1}^T$. 
Suppose that $\E\norm{\bg_t}^2 \leq G^2 + \sigma^2$ for all $t\in[T]$. 
Consider an online learner initialized with $\update_1 = \bm{0}$ and updated as follows:
\begin{align} \tag{$\beta$-\textsf{OMD}}\label{exp:domd-og}
 \update_{t+1} = \frac{\beta}{1 + \eta \mu} \left( \update_{t} - \eta \bg_t \right). 
\end{align}
Then, this online learner with  $\eta = \frac{2}{G + \sigma} \norm{\bu} \sqrt{1-\beta}$ achieves the following discounted regret bound:
\begin{align} \label{exp:discounted_regret}
 \E\left[\dregret_T(\bu)\right] \leq \frac{2\norm{\bu}(G + \sigma)}{\beta\sqrt{1-\beta}} + \frac{\mu}{2}\norm{\bu}^2.
\end{align}

\end{lemma}

\begin{proof}
See \autoref{pf:lem:omd-regret}.
\end{proof}

We expect that alternative online learning frameworks, such as follow-the-regularized-leader, can achieve the discounted regret bound similar to 
\eqref{exp:discounted_regret}.
However, a comprehensive exploration of discounted online learners lies outside the scope of this work. 

Using the bound \eqref{exp:discounted_regret}, we can derive an user-friendly nonconvex guarantee as follows. In the following two sections, only the regret bound \eqref{exp:discounted_regret} is important; the specifics of the discounted OMD algorithm are irrelevant. Any algorithm achieving a similar (or better) regret bound would provide the same (or improved) results.

\begin{thmbox}[{\bf Generic nonconvex guarantees}] \label{thm:opt}
Consider the iterates generated according to \autoref{alg:general}, where the online learner $\mathcal{A}$ achieves the discounted regret given by \eqref{exp:discounted_regret}. Let $\eps > 0$ be such that $\eps \leq \frac{7}{2}(G + \sigma)$ and let $T \geq 49 (G + \sigma)^2 \eps^{-2}$. 
Then, there exists a choice of parameters $\beta=\betas, D=\Ds, \mu=\mus$  such that with $\bby_t$ and $\tau$ as specified by \autoref{def:ema} and \autoref{def:random}, the following holds:
\begin{align*}
\E_{\tn }  \regnorm{\lambda}{\nabla F(\bby_{\tn})} &\leq  3\eps + \frac{ 1}{\Ds T}  \cdot \E \sum_{t=1}^{T} (F(\bx_{t}) - F(\bw_{t})) ,
\end{align*}
 The choice of parameters is summarized as follows:
\begin{itemize}[nosep]
\item The discounting factor is $\betas = 1 - (\frac{\eps}{7(G + \sigma)})^2$.
\item The comparator norm is $\Ds =    \frac{1}{4} \lambda^{-1/2} \eps^{1/2} \left(1+\frac{49(G + \sigma)^2 }{\eps^2} \sqrt{\cx} \right)^{-1}   $.
\item The regularization strength is $\mus = 2 \lambda^{1/2}\eps^{1/2}  \left(1 + \frac{49(G + \sigma)^2 }{\eps^2}\sqrt{\cx}  \right)$.
\end{itemize}
\end{thmbox}
\begin{proof}
See \autoref{pf:thm:opt}.
\end{proof}

Substituting the expression for $\Ds$ back into \autoref{thm:opt}, the upper bound becomes:
\begin{align}
    \E_{\tn }  \regnorm{\lambda}{\nabla F(\bby_{\tn})} \leq    3\eps +  \frac{4 \lambda^{1/2} \eps^{-1/2}}{T} \cdot  \left(1+\frac{49(G + \sigma)^2 }{\eps^2} \sqrt{\cx} \right)  \E \sum_{t=1}^{T} (F(\bx_{t}) - F(\bw_{t})) .
\end{align}
Hence, the main takeaway is that the nonconvex guarantees  depend on the magnitude of the following term:
\begin{align} \label{exp:main}
\left(1 + \frac{49(G + \sigma)^2 }{\eps^2} \sqrt{\cx} \right) \cdot \mathbb{E} \sum_{t=1}^{T} \left( F(\bx_t) - F(\bw_t) \right).
\end{align}
Thus, it is essential to select the iterates $\bx_t$ in a way that minimizes \eqref{exp:main}. As a warm-up, in the next section, we examine two simple special cases that lead to optimal complexity.

\section{Warm-up: two simple ways to achieve optimal nonconvex guarantees}
\label{sec:warmup}

In this section, as a warm-up, we present two simple ways to achieve the optimal nonconvex guarantee using our general conversion scheme, \autoref{alg:general}. Since \autoref{alg:general} maintains two sequences of iterates, $\bx_t$ and $\bw_t$, perhaps the two simplest options for $\bx_t$ are as follows:
\begin{enumerate}
    \item[I.] Set $\bx_t = \bw_{t-1}$ for all $t \in [T]$.
    \item[II.] Set $\bx_t = \bx_{t-1}$ for all $t \in [T]$.
\end{enumerate}

As we will demonstrate shortly, the first option recovers the previous conversion, while the second option leads to a novel conversion scheme, showcasing the versatility of our general framework. We will build on these warm-up cases to analyze schedule-free SGD in \autoref{sec:SF}.
Throughout this section,  $\mathcal{A}$ can be any online learner that achieves the discounted regret bound \eqref{exp:discounted_regret}.

\subsection{Option I leads to previous conversion}
 
\label{sec:momentum}

We begin with the first option, as described in \autoref{alg:momentum}.


\begin{algorithm}[H]
\caption{Option I} 
\label{alg:momentum}
\begin{algorithmic}[1]
\STATE In \autoref{alg:general}, set $\bx_t = \bw_{t-1}$ for all $t \in [T]$.  
\end{algorithmic}
\end{algorithm}
\autoref{alg:momentum} recovers previous approaches from \citep{cutkosky2023optimal}, \citep{zhang2024random}, and \citep{ahn2024adam}. 
In particular, since 
 \autoref{alg:momentum} leads to the update $\bw_{t} - \bw_{t-1} = \update_{t}$, it provides a nice interpretation of selecting the increments $\bw_{t} - \bw_{t-1}$ based on the online learner's output,  as highlighted by \citep{ahn2024understanding}.

The main advantage of \autoref{alg:momentum} is that the cumulative sum of loss decrements can be kept small due to a telescoping sum (recall that $\subopt \coloneqq F(\bx_0) - \inf_{\bx} F(\bx)$):
\begin{align}   \label{ineq:sumF}
\sum_{t=1}^{T} \left(F(\bx_t) - F(\bw_t)\right) &= \sum_{t=1}^{T} \left(F(\bw_{t-1}) - F(\bw_t)\right) = F(\bw_0) - F(\bw_T) \leq \subopt.
\end{align}

Moreover, from \autoref{alg:momentum}, we have $\norm{\bx_t - \bx_{t-1}} = \norm{\update_{t-1}}$ for $t \geq 2$, and for $t = 1$, $\|\bx_t - \bx_{t-1}\| = \|\bw_0 - \bx_0\| = 0$. Therefore, we have the following inequality:
\begin{align} \label{ineq:cx}
\mathbb{E} \sum_{t=1}^{T} \|\bx_t - \bx_{t-1}\|^2 &= \mathbb{E} \sum_{t=2}^{T} \|\update_{t-1}\|^2 \leq   \mathbb{E} \sum_{t=1}^{T} \|\update_t\|^2,
\end{align}
which shows that the iterate stability factor is at most 1, i.e., $\cx \leq 1$.

Combining these two calculations, \autoref{thm:opt} leads to the following nonconvex optimization guarantee.

\begin{corollary} \label{cor:momentum}
Consider the iterates of \autoref{alg:momentum}.
Under the setting of \autoref{thm:opt}, it holds that $\E_{\tn} \regnorm{\lambda}{\nabla F(\bby_{\tn})} \leq 4\eps$,
provided that 
\begin{align} \label{rate:optimal}
T \geq \max \left\{ 392(G+\sigma)^2 \subopt\lambda^{1/2} \eps^{-7/2} , ~~49 (G + \sigma)^2 \eps^{-2} \right\}.
\end{align}
\end{corollary}
\begin{proof}
   With \autoref{alg:momentum}, the above inequalities, \eqref{ineq:sumF} and \eqref{ineq:cx}, show that $\sum_{t=1}^{T} \left(F(\bx_t) - F(\bw_t)\right) \leq \subopt$ and $\cx \leq 1$. 
   Therefore, by applying \autoref{thm:opt}, we have 
\begin{align*}
\E_{\tau} \left[\regnorm{\lambda}{\nabla F(\bby_{\tau})}\right] &= 3\eps +  \frac{4 \lambda^{1/2} \eps^{-1/2}\left(1+\frac{49(G + \sigma)^2 }{\eps^2}   \right)}{T} \leq 3\eps +  \frac{4 \lambda^{1/2} \eps^{-1/2}\left( 2\cdot\frac{ 49(G + \sigma)^2 }{\eps^2}   \right)}{T} ,
\end{align*}
provided that $T \geq 49 (G + \sigma)^2 \eps^{-2}$. From this, the iteration complexity bound \eqref{rate:optimal} follows.
\end{proof}

We note that the complexity bound in \autoref{cor:momentum} is optimal, in light of the lower bound results of \citep{zhang2024random}.
In fact, our use of the random index $\tau$ leads to a slightly better guarantee; the upper bound in \citep{zhang2024random} includes the second argument of the maximum as $\OO{(G + \sigma)^3 \eps^{-3}}$, whereas ours is  $\OO{(G + \sigma)^2 \eps^{-2}}$.

We emphasize here that \autoref{alg:momentum} recovers commonly used momentum-based optimizers as special cases. \citet{zhang2024random} show that setting $\mathcal{A}$ as the OMD in \autoref{lem:omd-regret} corresponds to SGD with momentum, while \citet{ahn2024adam} demonstrate that choosing $\mathcal{A}$ as a discounted version of FTRL results in the Adam optimizer, up to minor modifications.

\subsection{Option II leads to a novel optimal conversion} 
\label{sec:anchor}

We now move to the second option. Since $\bx_t = \bx_{t-1}$ for all $t$, it follows that the $\bx$-iterates remain fixed, i.e., $\bx_t \equiv \bx_0$ for all $t \in [T]$.

\begin{algorithm}[H]
\caption{Option II  }
\label{alg:anchor}
\begin{algorithmic}[1]
\STATE In \autoref{alg:general}, set $\bx_t = \bx_0$ for all $t \in [T]$.
\end{algorithmic}
\end{algorithm}

The main advantage of \autoref{alg:anchor} is that the iterate stability factor is  $0$, \emph{i.e.}, $\cx = 0$. 
Hence, \autoref{thm:opt} implies the following result.

\begin{corollary} \label{cor:anchor}
Consider the iterates of \autoref{alg:anchor}.
Under the setting of \autoref{thm:opt}, the following holds:
\begin{align*}
\E_{\tn }  \regnorm{\lambda}{\nabla F(\bby_{\tn})}  \leq 3\eps +  \frac{4\lambda^{1/2} \eps^{-1/2}}{T}  \mathbb{E} \sum_{t=1}^{T} \left( F(\bx_0) - F(\bw_t) \right)   ,   
\end{align*}
provided that $T \geq 49 (G + \sigma)^2 \eps^{-2}$.
\end{corollary}

One way to utilize \autoref{cor:anchor} to achieve the optimal nonconvex guarantee is through the following anchoring scheme, which leverages multiple rounds of the algorithm.

\begin{algorithm}[H]
\caption{Anchoring scheme}
\label{alg:anchor_complete}
\begin{algorithmic}[1]
\STATE \textbf{Input:} Initial iterate $\bx_0$,  and integers $N, T \in \mathbb{N}$.
\STATE Set the initial anchor point $\anchor_1 \coloneqq \bx_0$.
\FOR{$n = 1, 2, \ldots, N$}
    \STATE Starting from $\anchor_n$, run \autoref{alg:anchor}  for $T$ iterations to generate the iterates $\{\bx^{(n)}_t, \bw^{(n)}_t, \by^{(n)}_t\}_{t=1}^T$. (By the choice in \autoref{alg:anchor}, we have $\bx^{(n)}_t \equiv \anchor_n$ for all $t\in[T]$.)
    \STATE Sample the next anchor point $\anchor_{n+1}$ uniformly at random from $\{\bw^{(n)}_t\}_{t=1}^T$.
\ENDFOR
\end{algorithmic}
\end{algorithm}

\begin{corollary} \label{cor:anchor_complete}
Consider the iterates of \autoref{alg:anchor_complete}.
Under the setting of \autoref{thm:opt}, it holds that 
\begin{align}
\E_{n \sim \Unif([N])}\E_{\tn} \regnorm{\lambda}{\nabla F(\bby^{(n)}_{\tn})} \leq 4\eps,
\end{align}
provided that $N \geq 4\subopt \lambda^{1/2} \eps^{-3/2}$ and $T \geq 49 (G + \sigma)^2 \eps^{-2}$.
\end{corollary}

\begin{proof}
Applying \autoref{cor:anchor_complete} to each epoch $n$, we obtain:
\begin{align}
    \E_{\tn }  \regnorm{\lambda}{\nabla F(\bby_{\tn}^{(n)})}  
    &\leq 3\eps +  \frac{4\lambda^{1/2} \eps^{-1/2}}{T}  \mathbb{E} \sum_{t=1}^{T} \left( F(\anchor_n) - F(\bw^{(n)}_t) \right) \\
    &  =  3\eps +   4\lambda^{1/2} \eps^{-1/2}   \mathbb{E}   \left( F(\anchor_n) - F(\anchor_{n+1})  \right),
\end{align}
where the last equality follows because $\anchor_{n+1}$ is chosen uniformly at random from $\{\bw^{(n)}_t\}_{t=1}^T$.

Summing over all epochs, since $\sum_{n=1}^N ( F(\anchor_n) - F(\anchor_{n+1})) =  F(\anchor_1) - F(\anchor_{N+1})$, we have:
\begin{align}
   \E_{n \sim \Unif([N])}\E_{\tn} \regnorm{\lambda}{\nabla F(\bby^{(n)}_{\tn})}   
   &\leq   3\eps +   \frac{4\lambda^{1/2} \eps^{-1/2}}{N}  \mathbb{E}   \left( F(\anchor_1) - F(\anchor_{N+1})  \right) \\
   &\leq 3\eps +   \frac{4\lambda^{1/2} \eps^{-1/2}\subopt}{N}.
\end{align}
Thus, setting $N \geq 4\subopt \lambda^{1/2} \eps^{-3/2}$ ensures that the right-hand side is at most $4\eps$.
\end{proof}

\autoref{cor:anchor} demonstrates that \autoref{alg:anchor_complete} provides an alternative approach for achieving the optimal nonconvex guarantee. Interestingly, this method bears a conceptual resemblance to the classic online-to-convex conversion~\citep{cesa2004generalization}. In particular, the traditional conversion runs an online learner for $T$ iterations and then selects an iterate uniformly at random (or averages iterates, applying Jensen's inequality). This approach closely parallels the procedure of a single epoch of \autoref{alg:anchor}. This can also be viewed as analogous to non-convex optimization approaches based upon repeatedly solving convex subproblems created by appropriate regularization, as discussed by \cite{chen2024open}.

Thus far, we explored two simple methods for selecting $\bx_t$, both of which can lead to the optimal nonconvex guarantee.  In the next section, we will consider yet another conversion approach that leads to the optimal guarantee.

\section{Schedule-free SGD is  effective for nonconvex optimization}
\label{sec:SF}

In this section, we build on \autoref{sec:warmup} and consider another special case of \autoref{alg:general} that achieves the optimal nonconvex guarantee. Specifically, we fix the online learner $\mathcal{A}$ to be the discounted OMD from \autoref{lem:omd-regret}, which we refer to as \ref{exp:domd}. This is in contrast to the previous sections in which the specifics of the online learner were not important. Recall from \autoref{lem:omd-regret} that the update for discounted OMD is:
\begin{align}  \tag{$\beta$-\textsf{OMD}}\label{exp:domd}
 \boxed{\update_{t+1} = \bt \left( \update_{t} - \eta \bg_t \right) \quad \text{where} \quad \bt \coloneqq \frac{\beta}{1 + \eta \mu}.}
\end{align}
Using the parameter settings $D_\star, \beta_\star \mu_\star$, specified in \autoref{thm:opt}, the optimal setting for $\eta$ according to \autoref{lem:omd-regret} is $\eta_\star = \frac{2}{G + \sigma} D_\star \sqrt{1-\beta_\star}$.

\subsection{Yet another approach to achieve the optimal nonconvex guarantee}

Consider the following special case of \autoref{alg:general}, specifically designed for  \ref{exp:domd}.

\begin{algorithm}[H]
\caption{Option III}
\label{alg:SF}
\begin{algorithmic}[1] 
\STATE  In \autoref{alg:general} with $\mathcal{A}$ chosen as \ref{exp:domd}, choose $\bx_t = \bx_{t-1} + \frac{1}{\bt} \update_t$ for all $t \in [T]$.   
\end{algorithmic}
\end{algorithm}

We first demonstrate that this conversion scheme achieves the optimal nonconvex guarantee. 
The key idea is that, with \autoref{alg:SF}, the iterates $\bx_t$ remain sufficiently close to $\bw_{t-1}$, allowing us to take advantage of the telescoping sum from \autoref{alg:momentum}.
More specifically, it follows that
\begin{align}
    \bx_t - \bw_{t-1} &= (\bx_{t-1} + \frac{1}{\bt}  \update_t) - (\bx_{t-1} + \update_{t-1}) = \frac{1}{\bt} \update_t - \update_{t-1}.
\end{align}
By the update rule of \ref{exp:domd}, a major cancellation occurs, and we have that $\frac{1}{\bt} \update_t - \update_{t-1}$ simplifies to $- \eta \bg_{t-1}$. Therefore, we can express:
\begin{align} \label{exp:diff_SF}
    \bx_t - \bw_{t-1} = - \eta \bg_{t-1}.
\end{align}
In other words, $\norm{\bx_t - \bw_{t-1}}$ is significantly smaller than the size of the update made by $\mathcal{A}$.
With this crucial \eqref{exp:diff_SF}, we can now show that \autoref{alg:SF} achieves the optimal nonconvex optimization guarantee.

\begin{corollary} \label{cor:SF}
Consider the iterates of \autoref{alg:SF} with the parameter choices as in \autoref{thm:opt}, \emph{i.e.}, $\beta = \betas$, $D = \Ds$, $\mu = \mus$, and the parameters of \ref{exp:domd} chosen as:
\begin{align}
\eta = \etas \coloneqq  \frac{2}{G + \sigma} \Ds \sqrt{1 - \betas}\quad \text{and} \quad    \bt= \bts \coloneqq \frac{\betas}{1 + \etas \mus}.
\end{align} 
Then, it holds that $\E_{\tn} \regnorm{\lambda}{\nabla F(\bby_{\tn})} \leq 5\eps$,
provided that 
\begin{align}  
T \geq \max \left\{ 980(G+\sigma)^2 \subopt\lambda^{1/2} \eps^{-7/2} , ~~49 (G + \sigma)^2 \eps^{-2} \right\}.
\end{align}
\end{corollary}
\begin{proof}[Proof sketch]
It is clear by definition that $\cx = \frac{1}{\bts}$. Thus, the crux of the proof lies in controlling the sum of loss decrements by leveraging \eqref{exp:diff_SF} in  order to apply \autoref{thm:opt}.
In particular, we decompose the sum using the $G$-Lipschitz continuity of $F$ as follows:
\begin{align}
    \sum_{t=1}^{T} \left(F(\bx_t) - F(\bw_t)\right) &= \sum_{t=1}^{T} \left(F(\bx_{t}) - F(\bw_{t-1}) + F(\bw_{t-1}) - F(\bw_t)\right) \\
    &\leq \sum_{t=1}^{T} G \|\bx_t - \bw_{t-1}\| + \sum_{t=1}^{T} \left(F(\bw_{t-1}) - F(\bw_t)\right).
\end{align}
Applying the closeness property \eqref{exp:diff_SF} between $\bx_t$ and $\bw_{t-1}$, we can show that the first term in the upper bound is at most $\OO{\Ds T \eps}$, making it a lower-order term. 
For a complete proof, see \autoref{pf:cor:SF}.
\end{proof}

\autoref{cor:SF} shows that \autoref{alg:SF} achieves the same optimal complexity as the conversion methods discussed in \autoref{sec:warmup}, differing only by multiplicative constant factors. Next, we will examine the update rule more closely.
It turns out that \autoref{alg:SF} corresponds exactly to schedule-free SGD.

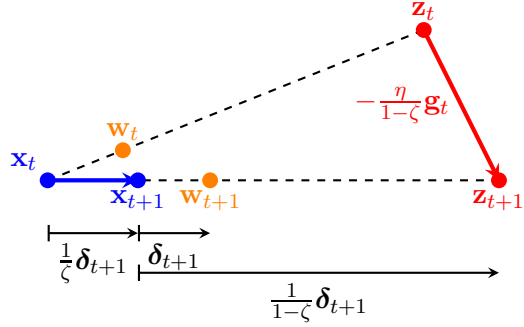
\begin{figure}
\begin{center}

\begin{tikzpicture}
\tikzset{
trajectoryz/.style={ultra thick, red, ->, >=stealth},
trajectoryx/.style={ultra thick, blue, ->, >=stealth},
trajectory/.style={thick, black, ->, >=stealth},
averaged/.style={thick, black, -, dashed, >=stealth},
distline/.style={thick, black, |->, >=stealth},  
}

\newcommand{\betav}{0.2} 
\newcommand{\width}{6}
\newcommand{\depth}{2}

\coordinate (Xt) at (0, 0);  
\coordinate (Zt) at ({\width-1}, \depth);
\coordinate (Wt) at ({(\width-1)*\betav}, \depth*\betav);
\coordinate (Xt1) at (\width*\betav,0);
\coordinate (Zt1) at (\width, 0);
\coordinate (Wt1) at ({\width*(2-\betav)*\betav}, 0);

\draw[averaged] (Xt) -- (Zt);
\draw[averaged] (Xt) -- (Zt1);
\draw[trajectoryz] (Zt) -- (Zt1) node[midway, left] {$-\frac{\eta}{1-\bt} \bg_{t}$};
\draw[trajectoryx] (Xt) -- (Xt1);


\draw[distline] ([yshift=-20pt]Xt) -- ([yshift=-20pt]Xt1) node[midway, below] {$\frac{1}{\bt} \update_{t+1}$};
\draw[distline] ([yshift=-20pt]Xt1) -- ([yshift=-20pt]Wt1) node[midway, below] {$ \update_{t+1}$};
\draw[distline] ([yshift=-35pt]Xt1) -- ([yshift=-35pt]Zt1) node[midway, below] {$\frac{1}{1-\bt} \update_{t+1}$};

\fill[blue] (Xt) circle (3pt) node[above left] {$\bx_{t}$};
\fill[blue] (Xt1) circle (3pt) node[below] {$\bx_{t+1}$};
\fill[orange] (Wt) circle (3pt) node[above] {$\bw_{t}$}; 
\fill[red] (Zt) circle (3pt) node[above] {$\bz_{t}$};
\fill[red] (Zt1) circle (3pt) node[below] {$\bz_{t+1}$}; 
\fill[orange] (Wt1) circle (3pt) node[below] {$\bw_{t+1}$}; 

\end{tikzpicture}
\end{center}
\caption{{\small Illustration of how \autoref{alg:SF} can be interpreted as schedule-free SGD. By defining the $\bz$-iterates according to \eqref{def:zt}, it becomes clear that the $\bz$-iterates follow the base SGD trajectory of schedule-free SGD \eqref{schefree}.}}

\label{fig:SF}
\end{figure}

\subsection{Option III is equivalent to schedule-free SGD}
\label{sec:SF_update}

A striking outcome of our general conversion framework (\autoref{alg:general}) is that one of its special cases, \autoref{alg:SF}, turns out to be equivalent to the schedule-free SGD method \eqref{schefree}. 

To demonstrate this equivalence, we introduce a set of extrapolated iterates, $\bz_t$, defined as follows:
\begin{align} \label{def:zt}
    \bz_t \coloneqq  \bx_t +\frac{1}{1-\bt} \update_t.
\end{align}
With these $\bz$-iterates, it becomes clear that the $\bz_t$ iterates follow the exact trajectory of the base SGD method used in schedule-free SGD, with an effective step size of $\gamma = \frac{\eta}{1-\bt}$. See \autoref{pf:prop:extrapolate} for details.

\begin{proposition}
\label{prop:extrapolate}
Under \autoref{alg:SF} and using the extrapolated iterates defined in \eqref{def:zt}, the following holds:
\begin{align*}
\bz_{t+1} - \bz_{t} = -\frac{\eta}{1-\bt}  \bg_{t}.
\end{align*}
\end{proposition} 

Thus, by defining these extrapolated iterates, we demonstrate that \autoref{alg:SF} mirrors the update rule of schedule-free SGD. We begin by presenting the explicit form of \autoref{alg:SF}, where we substitute the choice of $\bx_t$ according to \autoref{alg:SF}, \emph{i.e.}, $\bx_t = \bx_{t-1} + \frac{1}{\bt} \update_t$, along with the choice of $\mathcal{A}$ as \ref{exp:domd}. The resulting explicit update form is shown in \autoref{alg:SF_explicit}.
  
\begin{minipage}{0.45\textwidth}
\begin{algorithm}[H]
\caption{Explicit form of \autoref{alg:SF}}
\label{alg:SF_explicit}
\begin{algorithmic}[1]  
\FOR{$t = 1, 2, \ldots, T$}
\STATE Receive $\update_{t}$ from $\mathcal{A} = $\ref{exp:domd}, \emph{i.e.}, 
$\update_t = \bt (\update_{t-1} - \eta \bg_{t-1})$.
\STATE Update $\bx_t = \bx_{t-1} + \frac{1}{\bt} \update_t$.
\STATE Update $\bw_{t} = \bx_{t} + \update_{t}$.
\STATE Set $\by_{t} = \bx_{t} + s_t \update_{t}$, where $s_{t}$ is drawn uniformly from $[0,1]$ \emph{i.i.d.}
\STATE Compute $\bg_{t} \leftarrow \stograd(\by_{t})$.
\STATE Send loss $\ell_{t}(\cdot) = \langle \bg_{t}, \cdot \rangle + \frac{\mu}{2} \norm{\cdot}^2$ to $\mathcal{A}$.
\ENDFOR   
\end{algorithmic}
\end{algorithm}
\end{minipage}\hfill\begin{minipage}{0.45\textwidth}
\begin{algorithm}[H]
\caption{Rewriting of \autoref{alg:SF} using the extrapolated $\bz$-iterates \eqref{def:zt}}
\label{alg:SF_rewrite}
\begin{algorithmic}[1] 
\STATE \textbf{Input:} Initial iterates $\bx_0 = \bz_0$.
\FOR{$t = 1, 2, \ldots, T$}
\STATE Update $\bx_{t} = \bt \bx_{t-1} + (1-\bt) \bz_{t}$.
\STATE Set $\by_{t} = \kappa_t \bx_t + (1-\kappa_t) \bz_t$, where $\kappa_t$ is drawn uniformly from $[\bt, 1]$, \emph{i.i.d.}
\STATE Compute $\bg_{t} \leftarrow \stograd(\by_{t})$.
\STATE Update $\bz_{t+1} = \bz_{t} - \gamma \bg_{t}$, where the step size is chosen as $\gamma =\frac{\eta}{1-\bt}$.
\ENDFOR   
\end{algorithmic}
\end{algorithm}
\end{minipage}  
\vspace{15pt}

We can observe that \autoref{alg:SF_explicit} can be reformulated in terms of the iterates $\bx_t$, $\by_t$, and $\bz_t$, eliminating the dependence on $\bw_t$. This follows from \eqref{def:zt}, along with the choice $\bx_t = \bx_{t-1} + \frac{1}{\bt} \update_t$, which implies the following relationship:
\begin{align}
    \bx_t = \bt \bx_{t-1} + (1-\bt) \bz_t.
\end{align}
This result is also visually illustrated in \autoref{fig:SF}. Additionally, since $\by_t = \bx_t + s_t \update_t$, we have:
\begin{align}
    \by_t = \bx_t + s_t (1-\bt) (\bz_t - \bx_t) = \big(1 - s_t (1-\bt)\big) \bx_t + s_t (1-\bt) \bz_t.
\end{align}
By combining these steps, we obtain \autoref{alg:SF_rewrite}, a reformulation of \autoref{alg:SF_explicit}. Notably, this algorithm is equivalent to the schedule-free SGD method. Specifically, \autoref{alg:SF_rewrite} selects $\kappa_t$ uniformly from the interval $[\bt, 1]$ at each iteration, employs a step size $\gamma = \frac{\eta}{1-\bt}$, and consistently sets $c_t \equiv 1-\bt$ for all $t$.

\subsection{Practical insights from our results}
\label{sec:insights}

 
 Our results highlight an important property of schedule-free SGD: it not only achieves the optimal convex guarantee established by \citet{defazio2024road} but also attains the optimal nonconvex guarantee. This versatility helps explain the empirical success of schedule-free methods across a broad spectrum of optimization problems. 
It is important to note that our current analysis requires distinct parameter settings for $\gamma$ and $c_t$, depending on whether $F$ is convex or nonconvex.

We now discuss how our results offer new insights into parameter selection for schedule-free SGD. To begin, we note the following fact (see \autoref{pf:prop:parameters} for further details).

\begin{proposition} \label{prop:parameters}
    With the parameter choices given in \autoref{cor:SF}, we have $\bt = \bts = 1 - \Theta \left(\frac{\eps^2}{(G + \sigma)^2} \right)$. Specifically, the parameters in schedule-free SGD are set as follows:
    \begin{itemize}
        \item $\kappa_t$ is chosen uniformly from $[\bts, 1]$, implying that $1 - \Theta \left(\frac{\eps^2}{(G + \sigma)^2}\right) \leq \kappa_t \leq 1$ for all $t$. This selection ensures that $\kappa_t$ remains close to $1$.
        \item The step size for the base SGD ($\bz$-trajectory) is set to $\frac{\etas}{1-\bts}$, making it $\Theta\left(\frac{(G + \sigma)^2}{\eps^2}\right)$ times larger than the OMD step size, $\etas$.
    \end{itemize}
\end{proposition} 

\noindent Next, we interpret these parameter choices in light of empirical findings by \citet{defazio2024road} that lacked theoretical explanation.

First, their convex guarantee \cite[Theorem 2]{defazio2024road} permits $\kappa_t$ to be chosen arbitrarily within  the interval $[0,1]$, yet experimental results indicated that selecting $\kappa_t$ near $1$ (e.g., $0.98$) was crucial for strong empirical performance. Our results provides theoretical support for this choice,  as noted in the first bullet point of \autoref{prop:parameters}.

Another key empirical observation is that  the optimal learning rates for schedule-free variants exceeded those for standard base optimizers. They suggest that the ability to use larger learning rates without diverging may be a contributing factor to the faster convergence of schedule-free methods. Our results support this empirically observed benefit, as we show that the learning rate for nonconvex optimization is $\Theta (\frac{(G + \sigma)^2}{\eps^2} )$ times larger than the optimal step size for SGD momentum for this setting.

\section{Discussion}
\label{sec:discussion}

Motivated by the impressive empirical performance of schedule-free methods, this work investigates their effectiveness for nonconvex optimization. As a first step, we demonstrate that schedule-free SGD achieves optimal iteration complexity for nonsmooth, nonconvex optimization. This is accomplished through a general conversion framework that not only recovers existing conversions but also introduces two novel conversion schemes. Notably, one of these novel conversions directly corresponds to schedule-free SGD, which serves as the basis for our analysis. While this paper lays important groundwork, it merely scratches the surface of our understanding of schedule-free methods and opens up several avenues for future research.
Below, we outline a few of these potential directions for the reader's interest.

\paragraph{Other special cases of our general conversion.}  
In this work, we explore three specific instances of the general conversion framework. However, it is unlikely that these are the only viable conversions. Since these special cases yield highly practical optimizers, it would be worthwhile to investigate additional special cases and assess the practical implications of those conversions.

\paragraph{Adaptive schedule-free methods.}  
Considering the impressive practical performance of the schedule-free version of Adam, as highlighted in \citep{defazio2024road}, it would be intriguing to explore whether this method can be understood as a special case of the general online-to-nonconvex conversion framework. Our current analysis provides a correspondence only for schedule-free SGD.

\paragraph{Advanced weighting schemes.} The convex analysis of \citep{defazio2024road} allows for an arbitrary sequence of ``weights'' for each example which inform the choice for $c_t$ in \eqref{schefree}. While uniform weighting (corresponding to $c_t= 1/t$) is worst-case optimal, it is not \emph{instance} optimal, and in fact certain empirical heuristics such as learning rate warmup can be recovered by instance-dependent weighting \citep{defazio2023optimal}. Our analysis makes use of exponentially increasing weights, corresponding to a consant $c_t=\gamma$ as detailed in \autoref{alg:SF_rewrite}. While our weights achieve the optimal worst-case convergence guarantees, we conjecture that improvements are possible by incorporating instance-dependent weighting.

\paragraph{Truly universal methods.}  
It is noteworthy that the schedule-free method achieves optimal rates for nonsmooth losses regardless of their convexity. However, as discussed in \autoref{sec:insights}, the current analysis requires distinct parameter settings for $\gamma$ and $c_t$, depending on whether $F$ is convex or nonconvex. This indicates that we have not yet developed a fully unified algorithm that seamlessly addresses both cases. Nevertheless, these findings suggest the potential for a unified algorithmic framework. It would be valuable to explore whether a single parameter choice can be effective for both scenarios and to determine if such choices align with those observed in practice.

\paragraph{Why is schedule-free schedule-free?}  Lastly, we acknowledge that the main limitation of our analysis is its inability to fully explain why schedule-free methods can effectively alleviate the need for learning rate decay schedules, which is their most notable empirical advantage. Addressing this gap likely requires the development of a theoretical framework for understanding learning rate scheduling in nonconvex optimization.

\bibliographystyle{abbrvnat}
\bibliography{ref}

\newpage 

\appendix
\renewcommand{\appendixpagename}{\centering \LARGE Appendix}
\appendixpage
\startcontents[section]
\printcontents[section]{l}{1}{\setcounter{tocdepth}{2}}

\section{Proof of the discounted-to-nonconvex conversion (\autoref{lem:master})}
\label{pf:lem:master}

We begin this proof with some notations. 
Recall  that  $\tn$ is the random index distributed over $[T]$ as:
\begin{align}
\Pr(\tn =t ) =p_t \coloneqq \begin{cases}
\frac{1-\beta^t}{T}, & \text{if}~~ t=1,\dots, T-1,\\
\frac{1}{1-\beta}\cdot \frac{1-\beta^T  }{T},  & \text{if}~~ t=T.
\end{cases}
\end{align} 
For each $t \in [T]$, let $\eby_t$ be the random iterate distributed over $\{\by_s\}_{s=1}^t$ as:
\begin{align*}
\Pr(\eby_t= \by_s) = q_{t,s} \coloneqq \beta^{t-s} \cdot \frac{1-\beta}{1-\beta^t} \quad \text{for $s=1,2,\dots, t$.}
\end{align*} 
In particular, it follows that 
$\E_{\eby_t}[\eby_t] = \frac{1-\beta}{1-\beta^t} \cdot \sum_{s=1}^t\beta^{t-s} \by_s = \bby_t$.

In order to prove \autoref{lem:master}, we want to upper bound the following quantity:
\begin{align}
\E_{\tn} \regnorm{\lambda}{\nabla F(\bby_{\tn})}.
\end{align}
From the definition of a $(\lambda,\eps)$-stationary point (\autoref{def:reg_station}), it follows that
\begin{align}
\E_{\tn} \regnorm{\lambda}{\nabla F(\bby_{\tn})}  \leq   \E_{\tn }  \left[  \norm{\E_{\eby_{\tn}}\nabla F(\eby_{\tn})}  +\lambda  \E_{\eby_{\tn}} \norm{\eby_{\tn} - \bby_{\tn}}^2 \right].
\end{align} 
We will now upper bound each term individually. We begin with the first term. 

\begin{lemma} \label{lem:first_term}
For  $\beta\in(0,1)$, consider the iterates generated as per \autoref{alg:general}.
Consider the following definitions:
\begin{itemize}
\item For each $t \in [T]$, let   $\bu_t \coloneqq -D \frac{\sum_{s=1}^t\beta^{t-s}\nabla F(\by_s)}{\norm{\sum_{s=1}^t\beta^{t-s}\nabla F(\by_s)}}$. 
\item For each $t \in [T]$, let $\eby_t$ be the random iterate distributed over $\{\by_s\}_{s=1}^t$ as:
\begin{align*}
\Pr(\eby_t= \by_s) = q_{t,s} \coloneqq \beta^{t-s} \cdot \frac{1-\beta}{1-\beta^t} \quad \text{for $s=1,2,\dots, t$.}
\end{align*} 
\end{itemize}Then, the following upper bound holds:
\begin{align*} 
\E_{\tn }   \norm{\E_{\eby_{\tn}}\nabla F(\eby_{\tn})}  &\leq  \frac{\beta   \E\left[\dregret_T(\bu_T) \right] +  (1-\beta)\sum_{t=1}^{T}     \E\left[\dregret_t(\bu_t)\right]}{DN}  - \frac{\mu\E \sum_{t=1}^T   \norm{\update_t}^2}{2DT}  \\
&\quad +\frac{\E \sum_{t=1}^{T} (F(\bx_{t}) - F(\bw_{t}))}{DT}     + \frac{\sigma \beta}{T\sqrt{1-\beta}}  + \sigma \sqrt{1-\beta}+     \frac{\mu D}{2}  .
\end{align*} 
\end{lemma}
\begin{proof}
See \autoref{pf:lem:first_term}.
\end{proof}


For the second term, we use the following upper bound.

\begin{lemma}
\label{lem:second_term}
For  $\beta\in(0,1)$, consider the iterates generated as per \autoref{alg:general}
\begin{align*}
\E_{\tn}  \E_{\eby_{\tn}} \norm{\eby_{\tn} - \bby_{\tn}}^2  \leq   
\frac{2}{T} \E\sum_{t=1}^{T} \norm{\update_t}^2 +  \frac{4\beta }{(1-\beta)^2T } \E\sum_{t=1}^T \norm{\bx_t-\bx_{t-1}}^2.
\end{align*}
\end{lemma}
\begin{proof}
See \autoref{pf:var_gen}.
\end{proof}
Now let us combine above two results to finish the proof.
First, by \autoref{lem:second_term}, we get 
\begin{align*}
\lambda  \E_{\tn}\E_{\eby_{\tn}} \norm{\eby_{\tn} - \bby_{\tn}}^2 &\leq \frac{4\lambda}{T} \left(   \E\sum_{t=1}^{T} \norm{\update_t}^2 + \frac{1}{(1-\beta)^2}\E\sum_{t=1}^T \norm{\bx_t-\bx_{t-1}}^2\right) .
\end{align*} 
Therefore, Upon combining \autoref{lem:first_term} together with \autoref{lem:second_term}, we get: 
\begin{align*}
\E_{\tn }  \regnorm{\lambda}{\nabla F(\bby_{\tn})}  &\leq  \frac{ 1}{DT}\left(\beta \E\left[\dregret_T(\bu_T) \right] +  (1-\beta)\sum_{t=1}^{T}     \E\left[\dregret_t(\bu_t)\right]\right)    \\
&\quad + \frac{1}{DT} \E \sum_{t=1}^{T} (F(\bx_{t}) - F(\bw_{t}))   \\
&\quad +  \frac{4\lambda}{T} \left( \E\sum_{t=1}^{T} \norm{\update_t}^2 + \frac{1}{(1-\beta)^2}\E\sum_{t=1}^T \norm{\bx_t-\bx_{t-1}}^2 \right) - \frac{\mu}{2DT}  \E \sum_{t=1}^T   \norm{\update_t}^2\\
&\quad +  \frac{\mu D}{2}    + \frac{\sigma }{T\sqrt{1-\beta}}  + \sigma \sqrt{1-\beta}.
\end{align*}  
Now,  using the definition of $\cx$ (see \autoref{def:stability}), it follows that
\begin{align}
&\frac{4\lambda}{T} \left( \E\sum_{t=1}^{T} \norm{\update_t}^2 + \frac{1}{(1-\beta)^2}\E\sum_{t=1}^T \norm{\bx_t-\bx_{t-1}}^2 \right) - \frac{\mu}{2DT}  \E \sum_{t=1}^T   \norm{\update_t}^2 \\
&\quad \leq  \left(  4\lambda  (1 +\cx (1-\beta)^{-2} ) - \frac{\mu}{2D} \right) \cdot \frac{1}{T} \E \sum_{t=1}^T   \norm{\update_t}^2 \leq  0,
\end{align}
where the last line holds since  $\mu \geq  8\lambda D (1 +\cx (1-\beta)^{-2})$.
Therefore, we get the desirable upper bound.

\subsection{Proof of \autoref{lem:first_term}}
\label{pf:lem:first_term}
We begin with the following identity:
\begin{align*}
\sum_{n=1}^{T} \sum_{t=1}^{n} \beta^{n-t} (1 - \beta) (F(\bw_{t}) - F(\bx_{t}))&=   \sum_{t=1}^T \sum_{n=t}^T \beta^{n-t} (1 - \beta) (F(\bw_{t}) - F(\bx_{t}))\\
&= \sum_{t=1}^T (1 - \beta ^{T-t+1}) (F(\bw_t) - F(\bx_t)) \\
&= \sum_{t=1}^T (F(\bw_t) - F(\bx_t))  - \sum_{t=1}^T \beta^{T-t+1} (F(\bw_t) - F(\bx_t)) \,.
\end{align*}
After rearranging and taking expectations, it follows that 
\begin{align*}
0= \underbrace{ (1 - \beta)  \E \sum_{n=1}^T \sum_{t=1}^{n} \beta^{n-t}(F(\bw_{t}) - F(\bx_{t}))}_{=:\textbf{\textup{(A)}}}+ \underbrace{\beta \E \sum_{t=1}^T \beta^{T-t} (F(\bw_{t}) - F(\bx_{t}))}_{=:\textbf{\textup{(B)}}} + \E \sum_{t=1}^{T} (F(\bx_{t}) - F(\bw_{t}))     \,.
\end{align*}
Given this decomposition, we will carefully upper bound both \textbf{\textup{(A)}} and \textbf{\textup{(B)}}. 
This will be done based on the following result. 

\begin{lemma} \label{lem:backbone}
For  $\beta\in(0,1)$, consider the iterates generated as per \autoref{alg:general}.
Consider the following definitions:
\begin{itemize}
\item For each $t \in [T]$, let   $\bu_t \coloneqq -D \frac{\sum_{s=1}^t\beta^{t-s}\nabla F(\by_s)}{\norm{\sum_{s=1}^t\beta^{t-s}\nabla F(\by_s)}}$. 
\item For each $t \in [T]$, let $\eby_t$ be the random iterate distributed over $\{\by_s\}_{s=1}^t$ as:
\begin{align*}
\Pr(\eby_t= \by_s) = q_{t,s} \coloneqq \beta^{t-s} \cdot \frac{1-\beta}{1-\beta^t} \quad \text{for $s=1,2,\dots, t$.}
\end{align*} 
\end{itemize}
Then, for any $n\in [T]$, we have 
\begin{align*}
\E\sum_{t=1}^{n} \beta^{n-t}(F(\bw_{t}) - F(\bx_{t})) &\leq    -D \frac{1-\beta^n}{1-\beta}  \E \norm{\E_{\eby_n}\nabla F(\eby_n)} +\frac{\sigma D}{\sqrt{1-\beta}} \\
&\quad +\E
[\dregret_n(\bu_n)]   +\E   \sum_{t=1}^n   \beta^{n-t} ( - \frac{\mu}{2}\norm{\update_t}^2 + \frac{\mu}{2}D^2).
\end{align*}
\end{lemma}
\begin{proof}
See \autoref{pf:lem:backbone}.
\end{proof}

By \autoref{lem:backbone}, it follows that 
\begin{align}
\textbf{\textup{(A)}}  &\leq  -D \E \sum_{n=1}^T (1-\beta^n) \norm{\E_{\eby_n}\nabla F(\eby_n)}  +\sigma D T \sqrt{1-\beta} + (1-\beta) \E \sum_{n=1}^T \dregret_n(\bu_n)\\
&\quad +(1-\beta) \E  \sum_{n=1}^T \sum_{t=1}^n   \beta^{n-t} ( - \frac{\mu}{2}\norm{\update_t}^2 + \frac{\mu}{2}D^2).
\end{align}
Note that the last term can be further simplified:
\begin{align*}
(1-\beta) \E  \sum_{n=1}^T\sum_{t=1}^n   \beta^{n-t} ( - \frac{\mu}{2}\norm{\update_t}^2 + \frac{\mu}{2}D^2) &= (1-\beta) \E  \sum_{t=1}^T\sum_{n=t}^T   \beta^{n-t} ( - \frac{\mu}{2}\norm{\update_t}^2 + \frac{\mu}{2}D^2)\\
&=  \E  \sum_{t=1}^T   (1- \beta^{T-t+1}) ( - \frac{\mu}{2}\norm{\update_t}^2 + \frac{\mu}{2}D^2).
\end{align*}
Next, again using \autoref{lem:backbone}, we get 
\begin{align*}
\textbf{\textup{(B)}}  &\leq  - D \frac{\beta-\beta^{T+1}}{1-\beta}  \E \norm{\E_{\eby_T}\nabla F(\eby_T)} +\frac{\sigma \beta D}{\sqrt{1-\beta}} \\
&\quad +\beta \E
[\dregret_T(\bu_T)]   +\E   \sum_{t=1}^T   \beta^{T-t+1} ( - \frac{\mu}{2}\norm{\update_t}^2 + \frac{\mu}{2}D^2).
\end{align*}
Combining the above two inequalities for \textbf{\textup{(A)}} and \textbf{\textup{(B)}}, and plugging them back to the original identity, we get the following inequality:
\begin{align*}
0&\leq -D \cdot \E\left[ \frac{\beta-\beta^{T+1}}{1-\beta}   \norm{\E_{\eby_T}\nabla F(\eby_T)} + \sum_{t=1}^T (1-\beta^t) \norm{\E_{\eby_t}\nabla F(\eby_t)} \right] \\
&\quad + \sigma D T \sqrt{1-\beta} + (1-\beta) \E \sum_{t=1}^T \dregret_t(\bu_t)  +\E  \sum_{t=1}^T   (1- \beta^{T-t+1}) ( - \frac{\mu}{2}\norm{\update_t}^2 + \frac{\mu}{2}D^2)\\
&\quad +\frac{\sigma \beta D}{\sqrt{1-\beta}}  +\beta \E
[\dregret_T(\bu_T)]   +\E   \sum_{t=1}^T   \beta^{T-t+1} ( - \frac{\mu}{2}\norm{\update_t}^2 + \frac{\mu}{2}D^2) + \sum_{t=1}^T (F(\bx_t) - F(\bw_t)).
\end{align*}
Hence, after rearranging, we get the following inequality:
\begin{align*}
&D \cdot \E\left[ \frac{\beta-\beta^{T+1}}{1-\beta}   \norm{\E_{\eby_T}\nabla F(\eby_T)} + \sum_{t=1}^N (1-\beta^t) \norm{\E_{\eby_t}\nabla F(\eby_t)} \right]   \\
&\quad \leq  (1-\beta) \E \sum_{t=1}^T \dregret_t(\bu_t)+ \beta \E
[\dregret_T(\bu_T)] + \sum_{t=1}^T (F(\bx_t) - F(\bw_t))   \\
&\quad\quad  +\sigma D T \sqrt{1-\beta}  +\frac{\sigma \beta D}{\sqrt{1-\beta}}  +   +\E   \sum_{t=1}^T     ( - \frac{\mu}{2}\norm{\update_t}^2 + \frac{\mu}{2}D^2).
\end{align*}
In order to simplify the left hand side, notice first that 
\begin{align*}
\frac{\beta-\beta^{T+1}}{1-\beta} + \sum_{t=1}^T (1-\beta^t) = T +  \frac{\beta-\beta^{T+1}}{1-\beta} -  \frac{\beta-\beta^{T+1}}{1-\beta} =T.
\end{align*}
Let $\tn$ be the random index among $[T]$ such that
\begin{align}
\Pr(\tn = t ) = \begin{cases}
\frac{1-\beta^t}{T} & \text{if}~~ t=1,\dots, T-1,\\
\frac{1-\beta^T +\frac{\beta(1-\beta^T)}{1-\beta}}{T} =  \frac{1}{1-\beta}\cdot\frac{1-\beta^T}{T}  & \text{if}~~ t=T.
\end{cases}
\end{align}
Then, 
it follows that 
\begin{align}
D \cdot \E\left[ \frac{\beta-\beta^{T+1}}{1-\beta}   \norm{\E_{\eby_T}\nabla F(\eby_T)} + \sum_{t=1}^T (1-\beta^t) \norm{\E_{\eby_t}\nabla F(\eby_t)} \right]  = DT\cdot     \E_{\tn }   \norm{\E_{\eby_{\tn}}\nabla F(\eby_{\tn})}  .
\end{align}
Plugging this back to the above inequality and dividing both side by $D T$, we get the desired inequality in \autoref{lem:master}.

\subsection{Proof of \autoref{lem:second_term}}
\label{pf:var_gen}

Similar to $\eby_n, \bby_n$, we  define the random iterates for the $\bx$-sequence:
\begin{itemize}
\item For each $n \in [T]$, define the random iterate $\ebx_n$ distributed over $\{\bx_n\}_{s=1}^n$ as:
\begin{align*}
\Pr(\ebx_n= \bx_s) = q_{n,s} =\beta^{n-s} \cdot \frac{1-\beta}{1-\beta^n} \quad \text{for $s=1,2,\dots, n$.}
\end{align*} 
\item Let $\bbx_n$ be defined as:
\begin{align}
\bbx_n \coloneqq \E_{\ebx_n}[\ebx_n] = \frac{1-\beta}{1-\beta^n} \cdot \sum_{s=1}^n\beta^{n-s} \bx_s.  
\end{align}
\end{itemize}
Given this notation, one can prove the following result.s
\begin{lemma}[Variance decomposition lemma]
\label{lem:var_decomp}
For each $n \in [T]$ the following holds:
\begin{align*}
\E_{\eby_{n}} \norm{\eby_{n} - \bby_{n}}^2 \leq 2\E_{\ebx_{n}} \norm{\ebx_{n} - \bbx_{n}}^2 +  2\sum_{s=1}^{n} q_{n,s} \norm{\by_s - \bx_s}^2.
\end{align*}
\end{lemma}
\begin{proof} See \autoref{pf:lem:var_decomp}.
\end{proof}
Applying \autoref{lem:var_decomp} to the random index $\tn$ results in 
\begin{align*}
\E_{\tn}\E_{\eby_{\tn}} \norm{\eby_{\tn} - \bby_{\tn}}^2 \leq 2 \E_{\tn}\E_{\ebx_{\tn}} \norm{\ebx_{\tn} - \bbx_{\tn}}^2 + 2\E_{\tn} \left[ 
\sum_{s=1}^{\tn} q_{\tn,s} \norm{\by_s - \bx_s}^2 \right].
\end{align*}
For upper bounding the first term, we use the following result.
\begin{lemma} 
\label{lem:var1}
The following holds:
\begin{align*}
\E_{\tn}  \E_{\eby_{\tn}} \norm{\ebx_{\tn} - \bbx_{\tn}}^2  \leq    \frac{2\beta }{(1-\beta)^2 T} \sum_{n=1}^T \norm{\bx_n-\bx_{n-1}}^2.
\end{align*}
\end{lemma} 
\begin{proof}
See \autoref{pf:lem:var1}.
\end{proof}

Using \autoref{lem:var1} for the first term on the RHS, we obtain:
\begin{align*}
2 \E_{\tn}\E_{\ebx_{\tn}} \norm{\ebx_{\tn} - \bbx_{\tn}}^2 \leq 
\frac{4\beta }{(1-\beta)^2 T} \sum_{n=1}^T \norm{\bx_n-\bx_{n-1}}^2.
\end{align*}
Expanding the expectation for the second term, gives us:
\begin{align*}
2\E_{\tn} \left[
\sum_{s=1}^{\tn} q_{\tn,s} \norm{\bw_s - \bx_s}^2 \right] &= 
2 \sum_{n=1}^{T} \sum_{s=1}^{n} p_n q_{n,s} \norm{\bw_s - \bx_s}^2 \\
&= 2\sum_{s=1}^{T} \left (\sum_{n=s}^{T} p_n q_{n,s} \right) \norm{\bw_s - \bx_s}^2.
\end{align*}
Finally, plugging in the values of $p_{n}$ and $q_{n,s}$, we get:
\begin{align*}
\sum_{n=s}^{T} p_n q_{n,s} &= \sum_{n=s}^{T-1} \left(p_n q_{n,s} \right) + p_T q_{T,s}\\
&= \sum_{n=s}^{T-1} \left(\frac{1-\beta^n}{T} \cdot \left(\beta^{n-s} \frac{1-\beta}{1-\beta^n}\right)\right) +  \frac{1}{1-\beta} \frac{1-\beta^T}{T} \cdot \left( \beta^{T-s} \cdot \frac{1-\beta}{1-\beta^T} \right) \\
&= \frac{1}{T} \sum_{n=s}^{T-1} \beta^{n-s}(1-\beta) + \frac{\beta^{T-s}}{T} \\
&= \frac{1-\beta^{T-s}}{T}+\frac{\beta^{T-s}}{T} = \frac{1}{T}.
\end{align*}
Putting these two upper bounds together and using the fact that 
$\E\norm{\bw_s-\bx_s}^2 =  \E\norm{\update_s}^2$, our final bound follows.

\subsection{Proofs of technical lemmas}

\subsubsection{Proof of \autoref{lem:backbone}}
\label{pf:lem:backbone}

For an index $t$, let us first consider the term $\E[F(\bw_t) - F(\bx_t)]$. 
By \autoref{assump} and the fact that $\bw_t - \bx_t = \update_t$, together with the fact that $\by_{t} = \bx_{t} + s_t \update_{t}$ for $s_{t} \sim \mathrm{Unif}[0,1]$ (see \autoref{alg:general}), we have the following:
\begin{align*}
\E_{s_t}[F(\bw_t) - F(\bx_t)] &= \E_{s_t} \langle \nabla F(\by_t), \update_t \rangle.
\end{align*}

On the other hand, for $n\geq t$, it holds under the randomness over stochastic gradient $\bg_t$ that 
\begin{align*}
\E_{\bg_t} 
\langle \nabla F(\by_t), \update_t \rangle &= \E_{\bg_t} 
\langle \nabla F(\by_t), \bu_n \rangle +  \E_{\bg_t} 
\langle \nabla F(\by_t), \update_t - \bu_n \rangle \\ 
&= \E_{\bg_t} 
\langle \nabla F(\by_t), \bu_n \rangle +  \E_{\bg_t} 
\langle \nabla F(\by_t) -\bg_t, \update_t - \bu_n \rangle + \langle  \bg_t, \update_t - \bu_n \rangle \\ 
&= \E_{\bg_t} 
\langle \nabla F(\by_t), \bu_n \rangle +
\E_{\bg_t} \langle \nabla F(\by_t) - \bg_t , -\bu_n \rangle +
\E_{\bg_t} \langle \bg_t, \update_t - \bu_n \rangle ,
\end{align*}
where the last line follows since
\begin{align*}
    \E_{\bg_t} 
\langle \nabla F(\by_t) -\bg_t, \update_t \rangle =  
\langle \E_{\bg_t}[\nabla F(\by_t) -\bg_t], \update_t \rangle =0.
\end{align*}

Upon taking the overall expectation, we have
\begin{align} \label{exp:diff}
\E[F(\bw_t) - F(\bx_t)] = \E   \left[
\langle \nabla F(\by_t), \bu_n \rangle +
\langle \nabla F(\by_t) - \bg_t , -\bu_n \rangle +
\langle \bg_t, \update_t - \bu_n \rangle \right].
\end{align}
Hence, it follows that
\begin{align}
\E\sum_{t=1}^{n} \beta^{n-t}(F(\bw_{t}) - F(\bx_{t}))  = \E\sum_{t=1}^{n} \beta^{n-t} \left[ 
\langle \nabla F(\by_t), \bu_n \rangle +
\langle \nabla F(\by_t) - \bg_t , -\bu_n \rangle +
\langle \bg_t, \update_t - \bu_n \rangle\right].
\end{align}
Let us consider each term on the right hand side one by one:
\begin{itemize}
\item Using the definition $\bu_n \coloneqq -D \frac{\sum_{t=1}^n\beta^{n-t}\nabla F(\by_t)}{\norm{\sum_{t=1}^n\beta^{n-t}\nabla F(\by_t)}}$, the first term can be expressed as:
\begin{align}
\E\sum_{t=1}^n \beta^{n-t}  \inp{\nabla F(\by_t
)}{\bu_n}  &=  \E \inp{\sum_{t=1}^n \beta^{n-t}\nabla F(\by_t
)}{-D\frac{\sum_{t=1}^n \beta^{n-t}\nabla F(\by_t
)}{\norm{\sum_{t=1}^n \beta^{n-t}\nabla F(\by_t
)}} } \\
&= -D \E \norm{\sum_{t=1}^n \beta^{n-t}\nabla F(\by_t
)} =    -D \frac{1-\beta^n}{1-\beta}  \E \norm{\E_{\eby_n}\nabla F(\eby_n)} \,.
\end{align} 
\item  For the second term, using Cauchy-Schwartz inequality, we have
\begin{align}
\E\sum_{t=1}^n \beta^{n-t}   \inp{\nabla F(\by_t
) - \bg_t}{-\bu_n} \leq \sqrt{ \E \norm{\sum_{t=1}^n \beta^{n-t}   (\nabla F(\by_t
) - \bg_t)}^2 \E\norm{\bu_n}^2 }\,.
\end{align}
Using the bounded variance assumption on the stochastic gradient oracle, we have
\begin{align}
\E \norm{\sum_{t=1}^n \beta^{n-t}   (\nabla F(\by_t
) - \bg_t)}^2 = \E \sum_{t=1}^n \beta^{2(n-t)}   \norm{\nabla F(\by_t
) - \bg_t}^2  \leq \frac{\sigma^2}{1-\beta^2} \leq \frac{\sigma^2}{1-\beta}\,.
\end{align}
where we used the fact $\frac{1}{1-\beta^2}\leq \frac{1}{1-\beta}$.
Thus, it holds that 
\begin{align*}
\E\sum_{t=1}^n \beta^{n-t}   \inp{\nabla F(\by_t
) - \bg_t}{-\bu_n} \leq \frac{\sigma D}{\sqrt{1-\beta}}. 
\end{align*} 

\item  For the third term, we have 
\begin{align*}
\E  \sum_{t=1}^{n} \beta^{n-t}  
\langle \bg_{t}, \update_{t} - \bu_n \rangle &= \E
[\dregret_n(\bu_n)]   +\E   \sum_{t=1}^n   \beta^{n-t} ( - \frac{\mu}{2}\norm{\update_t}^2 + \frac{\mu}{2}\norm{\bu_n}^2)\\
& = \E
[\dregret_n(\bu_n)]   +\E   \sum_{t=1}^n   \beta^{n-t} ( - \frac{\mu}{2}\norm{\update_t}^2 + \frac{\mu}{2}D^2).
\end{align*} 
\end{itemize}
This completes the proof of \autoref{lem:backbone}. \

\subsubsection{Proof of \autoref{lem:var_decomp}}
\label{pf:lem:var_decomp}

Let $S_n$ be a random index distributed over $1, 2, \dots, n$ such that $\Pr(S_n = s) = q_{n,s}$.
Moreover, let $\tS_n$ be an independent copy of $S_n$.
Then, it follows that   
\begin{align*}
\E_{\eby_{n}} \norm{\eby_{n} - \bby_{n}}^2 &= 
\E_{S_n} \norm{\by_{S_n} -\E_{\tS_n} [\by_{\tS_n}]}^2\\ 
&= \E_{S_n} \norm{\bx_{S_n} -\E_{\tS_n} [\bx_{\tS_n}] + \by_{S_n} - \bx_{S_n} - \E_{\tS_n}[\by_{\tS_n} - \bx_{\tS_n}]}^2 \\
&\leq 2\E_{S_n} \norm{\bx_{S_n} -\E_{\tS_n}[\bx_{\tS_n}]}^2 + 2 \E_{S_n} \norm{\by_{S_n} - \bx_{S_n} - \E_{\tS_n}[\by_{\tS_n} - \bx_{\tS_n}]}^2 \\
&\leq 2\E_{S_n} \norm{\bx_{S_n} -\E_{\tS_n}[\bx_{\tS_n}]}^2 + 2 \E_{S_n} \norm{\by_{S_n} - \bx_{S_n}}^2 \\
&=2\E_{\ebx_{n}} \norm{\ebx_{n} - \bbx_{n}}^2 +  2\sum_{s=1}^{n} q_{n,s} \norm{\by_s - \bx_s}^2.
\end{align*}
where for the first bound we have used $\norm{x+y}^2 \leq 2\norm{x}^2 + 2\norm{y}^2$, and for the second one we applied $\E \norm{\ebx - \E \ebx}^2 = \E \norm{\ebx}^2 - \norm{\E \ebx}^2 \leq \E \norm{\ebx}^2$.
This completes the proof.

\subsubsection{Proof of \autoref{lem:var1}}

\label{pf:lem:var1}

We start with the following two results.
\begin{proposition} \label{prop1:var1}
For each $n \in [T]$, it holds that 
\begin{align*}
\E_{\ebx_n} \norm{\ebx_n - \bbx_n}^2 \leq 2\sum_{t=1}^n \lambda_{n,t} \norm{\bx_{t}- \bx_{t-1}}^2, \quad \text{where}~\lambda_{n,t} \coloneqq\sum_{i=t}^n\sum_{j=1}^{t-1}  q_{n,i} q_{n,j}  (i-j) .
\end{align*} 
\end{proposition}
\begin{proof}[{\bf Proof of \autoref{prop1:var1}.}]
Let us fix $n \in [T]$. 
Let $\widehat{\ebx}_n$ be an independent copy of $\ebx_n$. Then, by Jensen's inequality, we get:
\begin{align}
\E_{\ebx_n} \norm{\ebx_n - \bbx_n}^2  &= 
\E_{\ebx_n} \norm{\ebx_n - \E_{\widehat{\ebx}_n}[\widehat{\ebx}_n]}^2  \leq  
\E_{\ebx_n, \widehat{\ebx}_n} \norm{\ebx_n - \widehat{\ebx}_n}^2  \\
&= 2 \sum_{i=1}^n \sum_{j=1}^{i-1} q_{n,i} q_{n,j} \norm{\bx_i- \bx_{j}}^2\,.  
\end{align}
Next, using Cauchy-Schwartz and triangle inequality, we get:
\begin{align*}
\norm{\bx_i- \bx_{j}}^2 \leq \left(\sum_{t=j+1}^i \norm{\bx_{t}-\bx_{t-1}}\right)^2 \leq (i-j) \sum_{t=j+1}^i \norm{\bx_t-\bx_{t-1}}^2.
\end{align*}
Hence, it follows that
\begin{align*}
\E_{\ebx_n} \norm{\ebx_n - \bbx_n}^2 &\leq 2 \sum_{i=1}^n \sum_{j=1}^{i-1}  \sum_{t=j+1}^i q_{n,i} q_{n,j}  (i-j)  \norm{\bx_t-\bx_{t-1}}^2 \\
&= 2 \sum_{t=1}^n \left(\sum_{i=t}^n\sum_{j=1}^{t-1}  q_{n,i} q_{n,j}  (i-j)  \right) \norm{\bx_t-\bx_{t-1}}^2 
\end{align*}
This completes the proof.    
\end{proof}

\begin{proposition} \label{prop2:var1}  
For $1\leq t\leq n\leq T$, it holds that 
\begin{align*}
\lambda_{n,t} \leq  \frac{(n-t+1) \beta^{n-t+1}}{1-\beta^n}.
\end{align*} 
\end{proposition}
\begin{proof}[{\bf Proof of \autoref{prop2:var1}.}]
In order to prove the desired upper bound, let us first simplify the expression for $\lambda_{n,t}$ as follows:
\begin{align*}
\lambda_{n,t} &= \left(\frac{1-\beta}{1-\beta^n}\right)^2 \cdot \sum_{i=t}^n\sum_{j=1}^{t-1}  \beta^{n-i} \beta^{n-j} (i-j)\\
&= \left(\frac{1-\beta}{1-\beta^n}\right)^2 \cdot \sum_{i=t}^n\sum_{\delta=i-t+1}^{i-1}  \beta^{n-i} \beta^{n-i+\delta} \delta &&(\delta\coloneqq i-j)\\
&= \left(\frac{1-\beta}{1-\beta^n}\right)^2 \cdot \sum_{k=0}^{n-t}\sum_{\delta=n-k-t+1}^{n-k-1}  \beta^{k} \beta^{k+\delta} \delta &&(k\coloneqq n-i)\\
&= \left(\frac{1-\beta}{1-\beta^n}\right)^2 \cdot \sum_{k=0}^{n-t}\left[\beta^{2k}\sum_{\delta=n-k-t+1}^{n-k-1}   \delta\beta^{\delta}  \right].
\end{align*}
Next, let us compute the inner summation:
\begin{align*}
\sum_{\delta=n-k-t+1}^{n-k-1}   \delta\beta^{\delta} &=  \beta \cdot \frac{\D}{\D\beta} \sum_{\delta=n-k-t+1}^{n-k-1}  \beta^{\delta} = \beta \cdot \frac{\D}{\D\beta} \left( \frac{\beta^{n-k-t+1} - \beta^{n-k}}{1-\beta}
\right) \\
&= \frac{\beta^{a-k+1} - \beta^{b-k+1}}{(1-\beta)^2} + \frac{(a-k) \beta^{a-k} -(b-k) \beta^{b-k}}{1-\beta}. &&(a \coloneqq n-t+1,~ b\coloneqq n)
\end{align*}
Substituting this back to the above expression for $\lambda_{n,t}$, we have
\begin{align*}
\lambda_{n,t}  &= \left(\frac{1-\beta}{1-\beta^n}\right)^2 \cdot \sum_{k=0}^{n-t} \beta^{2k} \left[ \frac{\beta^{a-k+1} - \beta^{b-k+1}}{(1-\beta)^2} + \frac{(a-k) \beta^{a-k} -(b-k) \beta^{b-k}}{1-\beta} \right]\\
&= \left(\frac{1-\beta}{1-\beta^n}\right)^2 \cdot \sum_{k=0}^{n-t}   \left[ \left(\frac{\beta^{a+1} -\beta^{b+1}}{(1-\beta)^2} + \frac{a\beta^a -b\beta^b}{1-\beta}\right) \beta^k   - \frac{\beta^a -\beta^b}{1-\beta}k\beta^k \right].
\end{align*}
Now one can compute the first and second terms above using the following identities:
\begin{itemize}
\item First term: note that $\sum_{k=0}^{n-t} \beta^k = \frac{1-\beta^{n-t+1}}{1-\beta} = \frac{1-\beta^a}{1-\beta}$.
\item Second term: note that $\sum_{k=0}^{n-t} k\beta^k = \beta \frac{\D}{\D \beta} \left( \frac{1-\beta^a}{1-\beta}\right) =\frac{\beta -\beta^{a+1}}{(1-\beta)^2} - \frac{a\beta^a}{1-\beta}$.
\end{itemize}
Thus, it follows that
\begin{align*}
\lambda_{n,t}   
&= \left(\frac{1-\beta}{1-\beta^n}\right)^2 \cdot   \left[ \left(\frac{\beta^{a+1} -\beta^{b+1}}{(1-\beta)^2} + \frac{a\beta^a -b\beta^b}{1-\beta}\right)\cdot  \frac{1-\beta^a}{1-\beta}   - \frac{\beta^a -\beta^b}{1-\beta} \cdot \left( \frac{\beta -\beta^{a+1}}{(1-\beta)^2} - \frac{a\beta^a}{1-\beta}\right)  \right]\\
&= \frac{ (\beta^{a+1}-\beta^{b+1}) \frac{1-\beta^a}{1-\beta} + (a \beta^a -b\beta^b)(1-\beta^a) - (\beta^{a+1}-\beta^{b+1}) \frac{1-\beta^a}{1-\beta }  + (\beta^a-\beta^b) a\beta^a}{(1-\beta^n)^2}\\
&= \frac{   (a \beta^a -b\beta^b)(1-\beta^a)  + (\beta^a-\beta^b) a\beta^a}{(1-\beta^n)^2} =  \frac{  a\beta^a (1-\beta^b) -b\beta^b(1-\beta^a)}{(1-\beta^n)^2} \leq \frac{  a\beta^a (1-\beta^b)  }{(1-\beta^n)^2}.
\end{align*}
Upon plugging back the definition of $a$ and $b$, this completes the proof of \autoref{prop2:var1}.
\end{proof}

For the remainder of the proof, let us use the following notation for simplicity:
\begin{align}
\delta_t \coloneqq \norm{\bx_{t} - \bx_{t-1}}.
\end{align}
Combining \autoref{prop1:var1} and \autoref{prop2:var1}, it follows that 
\begin{align} \label{bound:variance_n}
\E_{\ebx_n} \norm{\ebx_n - \bbx_n}^2 \leq 2\sum_{t=1}^n \frac{(n-t+1) \beta^{n-t+1}}{1-\beta^n} \delta_t^2.
\end{align}
With this bound, one can upper bound the overall averaged variance term as follows: 
\begin{align}
& \E_{\tn}  \E_{\ebx_{\tn}} \norm{\ebx_{\tn} - \bbx_{\tn}}^2 =  \sum_{n=1}^{T-1} \left[ p_n \E_{\ebx_n} \norm{\ebx_n - \bbx_n}^2 \right] + p_T \E_{\ebx_T} \norm{\ebx_T - \bbx_T}^2\\
&\quad\leq 2\sum_{n=1}^{T-1} \left[ \frac{1-\beta^n}{T} \sum_{t=1}^n \frac{(n-t+1) \beta^{n-t+1}}{1-\beta^n} \delta_t^2\right] +  \frac{2(1-\beta^T)}{(1-\beta)T}\sum_{n=1}^T \frac{(T-n+1) \beta^{T-n+1}}{1-\beta^T} \delta_n^2\\
&\quad \leq \frac{2}{T}\sum_{n=1}^{T-1}  \sum_{t=1}^n \left[   (n-t+1) \beta^{n-t+1}  \delta_t^2\right] +  \frac{2}{(1-\beta)T}\sum_{n=1}^T  (T-n+1) \beta^{T-n+1}  \delta_n^2.
\end{align}
For the first term above, note that
\begin{align*}
\sum_{n=1}^{T-1}  \sum_{t=1}^n \left[   (n-t+1) \beta^{n-t+1}  \delta_t^2\right] &= \sum_{t=1}^{T-1} \left[ \delta_t^2 \sum_{n=t}^{T-1}  (n-t+1) \beta^{n-t+1}  \right]\\
&= \sum_{t=1}^{T-1} \left[ \delta_t^2 \sum_{k=1}^{T-t}  k \beta^k  \right]\\
&= \sum_{t=1}^{T-1} \left[ \delta_t^2 \beta \frac{\D}{\D\beta}\left( \sum_{k=1}^{T-t}   \beta^k  \right)  \right]= \sum_{t=1}^{T-1} \left[ \delta_t^2 \beta \frac{\D}{\D\beta}\left( \frac{\beta -\beta^{T-t+1}}{1-\beta} \right)  \right]\\
&= \sum_{t=1}^{T-1} \left[ \delta_t^2    \left( \frac{\beta^2 -\beta^{T-t+2}}{(1-\beta)^2} + \frac{\beta -(T-t+1)\beta^{T-t+1}}{1-\beta}  \right)  \right] \\
&=  \sum_{n=1}^{T-1} \left(  \frac{\beta^2 (1- \beta^{T-n})}{(1-\beta)^2}
+ \frac{\beta}{1-\beta}  -\frac{(T-n+1)\beta^{T-n+1}}{1-\beta}      \right)  \delta_n^2.  
\end{align*}
Upon plugging this back to the upper bound on the overall averaged variance, we get:
\begin{align*}
\E_{\tn}  \E_{\ebx_{\tn}} \norm{\ebx_{\tn} - \bbx_{\tn}}^2  
&\leq \frac{2}{T} \sum_{n=1}^{T-1} \left[  \left( \frac{\beta^2 (1- \beta^{T-n})}{(1-\beta)^2} 
+ \frac{\beta }{1-\beta}  \right)   \delta_n^2 \right] +  \frac{2\beta}{(1-\beta)T}\delta_T^2 \\
&\leq \frac{2}{T} \sum_{n=1}^{T-1} \left[  \left( \frac{\beta^2 + \beta(1-\beta) }{(1-\beta)^2}  \right)   \delta_n^2 \right] +  \frac{2\beta}{(1-\beta)T}\delta_T^2 \\
&\leq \frac{2\beta }{(1-\beta)^2T} \sum_{n=1}^T \delta_n^2.
\end{align*}
This completes the proof of \autoref{lem:var1}.

\section{Proof of the user-friendly nonconvex guarantees (\autoref{thm:opt})}
\label{pf:thm:opt}

Let us first recall the upper bound from \autoref{lem:master}:
\begin{align*}
\E_{\tn} \left[ \regnorm{\lambda}{\nabla F(\bby_{\tn})} \right] &\leq \frac{1}{DT} \left( \beta \E\left[\dregret_T(\bu_T)\right] + (1-\beta) \sum_{t=1}^{T} \E\left[\dregret_t(\bu_t)\right] \right) \\
&\quad 
+ \frac{1}{DT} \E \left[ \sum_{t=1}^{T} (F(\bx_t) - F(\bw_t)) \right] 
+ \frac{\mu D}{2} + \frac{\sigma}{T\sqrt{1-\beta}} + \sigma \sqrt{1-\beta}.
\end{align*}
Below, we further upper bound the discounted regret terms.
First, since $\norm{\bu_n} = D$, \autoref{lem:omd-regret} implies:
\begin{align} \label{dregret}
 \forall t \in [T], \quad \beta \E\left[\dregret_t(\bu_t)\right] \leq \frac{2D(G+\sigma)}{\sqrt{1-\beta}} + \frac{\mu}{2} D^2, 
\end{align}
 where  this upper bound holds for all $t \in [T]$ since the upper bound does not depend on $T$.

To simplify the upper bound on the regret term, we impose the following condition:
\begin{align} \label{cond:beta}
\frac{1}{2} \leq \beta \leq 1 - \frac{1}{T}.
\end{align}
Additionally, we introduce the following notation to streamline our expressions:
\begin{align}
    \boxed{\alpha \coloneqq 1 - \beta.}
\end{align}
In the sequel, we will select $1 - \beta = O(\eps^2)$, so readers can treat $\alpha$ as a small quantity.

With the condition \eqref{cond:beta} together with the simplifying notation, we have $\beta^{-1} \leq 2$ and $\alpha T \geq 1$. Therefore, plugging the upper bound \eqref{dregret} into the discounted regret term, it follows that:
\[
\frac{1}{DT} \left( \beta \E\left[\dregret_T(\bu_T)\right] + \alpha \sum_{t=1}^{T} \E\left[\dregret_t(\bu_t)\right]\right) 
\leq \frac{(2\alpha T + 1)}{DT} \left( \frac{2D(G+\sigma)}{\sqrt{\alpha}} + \frac{1}{2}\mu D^2 \right)
\]
\[
\leq \frac{3\alpha T }{DT} \left( \frac{2D(G+\sigma)}{\sqrt{\alpha}} + \frac{1}{2} \mu D^2 \right)
= 6(G + \sigma) \sqrt{\alpha} + \frac{3}{2} \alpha \mu D \leq 6(G + \sigma) \sqrt{\alpha} + \frac{3}{2} \mu D.
\]

Now let us plug this bound back to \autoref{lem:master}.
For the regularization strength, we choose 
\begin{align*}
\mu =\mus \coloneqq 8 \lambda D \left(1 + \sqrt{\cx} \alpha^{-1}\right)^2 .
\end{align*}
This is a valid choice for \autoref{lem:master} since $8 \lambda D \left(1 + \sqrt{\cx} \alpha^{-1}\right)^2\geq     8 \lambda D \left(1 +  \cx \alpha^{-2}\right)$.
Thus, we arrive at the following bound:
\begin{align*}
\E_{\tn} \left[ \regnorm{\lambda}{\nabla F(\bby_{\tn})} \right] 
\leq \frac{\E \left[ \sum_{t=1}^{T} (F(\bx_t) - F(\bw_t)) \right]}{DT} + \underbrace{\frac{\sigma}{T\sqrt{\alpha}}}_{\textbf{(A)}} + \underbrace{7(G + \sigma) \sqrt{\alpha}}_{\textbf{(B)}} + \underbrace{16\lambda D^2 (1 + \sqrt{\cx} \alpha^{-1})^2}_{\textbf{(C)}}.
\end{align*}

Now, we show that  the choice of paramters in \autoref{thm:opt} (\emph{i.e.}, $\betas, \Ds$ defined in the statement)  leads to  each of the underbraced terms is bounded by $\eps$:
\begin{itemize}
    \item \textbf{(B):} Choosing $\beta =\betas  = 1- \left( \frac{\eps}{7(G + \sigma)}\right)^2$, we have $\alpha = \left( \frac{\eps}{7(G + \sigma)}\right)^2$.
    Hence, term \textbf{(B)} is at most $\eps$.
    \item \textbf{(C):} Taking $D = \Ds = \frac{1}{4}  \lambda^{-1/2} \eps^{1/2} (1+\sqrt{\cx} \alpha^{-1} )^{-1}$, we have
    \begin{align*}
        \textbf{(C)}  &= 16\lambda   (1 + \sqrt{\cx} \alpha^{-1})^2 \Ds^2  =\eps.
    \end{align*}
    Hence, term \textbf{(C)} is at most $\eps$. Expanding with $\alpha = \left( \frac{\eps}{7(G + \sigma)}\right)^2$, we have:
    \begin{align*}
    D = \Ds = \frac{1}{4} \lambda^{-1/2} \eps^{1/2} \left(1+\frac{49(G + \sigma)^2 }{\eps^2} \sqrt{\cx} \right)^{-1} .     
    \end{align*}
       \item \textbf{(A):} Taking $T \geq \sigma \alpha^{-1/2} \eps^{-1}$, term \textbf{(A)} is at most $\eps$.
\end{itemize}

Therefore, we have the following inequality as desired:
\[
\E_{\tn} \left[ \regnorm{\lambda}{\nabla F(\bby_{\tn})} \right] \leq 3\eps + \frac{\E \left[ \sum_{t=1}^{T} (F(\bx_t) - F(\bw_t)) \right]}{\Ds T}.
\]

Finally, let us reconcile these parameter choices with the earlier condition \eqref{cond:beta} on $\beta$. As long as $\eps \leq \frac{7}{2} (G + \sigma)$, it holds that $\alpha \leq \frac{1}{4}$, implying $\beta \geq \frac{3}{4}$, which satisfies the condition $\beta \geq \frac{1}{2}$. Additionally, to ensure $\beta \leq 1 - \frac{1}{T}$, we impose $T \geq \alpha^{-1}$. Therefore, the overall condition on $T$ becomes:
\[
T \geq \max\left\{\sigma \alpha^{-1/2} \eps^{-1},~~ \alpha^{-1}\right\}
= \max\left\{7\sigma (G + \sigma) \eps^{-2},~~ 49(G + \sigma)^2 \eps^{-2}\right\} = 49(G + \sigma)^2 \eps^{-2}.
\]
This completes the proof of \autoref{thm:opt}.

\section{Proof of \autoref{cor:SF}}
\label{pf:cor:SF}

We first recall the choice of parameters from \autoref{thm:opt} for the reader's convenience.

\begin{mdframed}
\begin{itemize}[nosep]
    \item The discount factor is $\betas = 1 - \left(\frac{\eps}{7(G + \sigma)}\right)^2$.
    \item The comparator norm is $\Ds = \frac{1}{4} \lambda^{-1/2} \eps^{1/2} \left(1 + \frac{49(G + \sigma)^2}{\eps^2} \sqrt{\cx} \right)^{-1}$.
    \item The regularization strength is $\mus = 2 \lambda^{1/2} \eps^{1/2} \left(1 + \frac{49(G + \sigma)^2}{\eps^2} \sqrt{\cx} \right)$.
    \item The step size of \ref{exp:domd} is $\etas = \frac{2}{G + \sigma} \Ds \sqrt{1 - \betas}$.
\end{itemize}
\end{mdframed}

We now show that $\cx$ is bounded by 4.
First, note that with the above choice of parameters, we have
\begin{align}
    \bt = \bts \coloneqq \frac{\betas}{1 + \etas \mus}.
\end{align}
Since $\eps \leq \frac{7}{2}(G + \sigma)$, it follows that $\betas \geq 1 - \frac{1}{4} = \frac{3}{4}$. Moreover, we have
\begin{align}
    \etas \mus &= \frac{2}{G + \sigma} \Ds \sqrt{1 - \betas} \cdot 2 \lambda^{1/2} \eps^{1/2} \left(1 + \frac{49(G + \sigma)^2}{\eps^2} \sqrt{\cx} \right) \\
    &= \frac{2}{G + \sigma} \cdot \frac{1}{4} \lambda^{-1/2} \eps^{1/2} \cdot \left( \frac{\eps}{7(G + \sigma)} \right) \cdot 2 \lambda^{1/2} \eps^{1/2} \\
    &= \frac{\eps^2}{7(G + \sigma)^2} \leq \frac{7}{4}.
\end{align}
This shows that
\begin{align}
    \bts^{-1} = \frac{1 + \etas \mus}{\betas} \leq \frac{11/4}{3/4} \leq 4.
\end{align}
Thus, from \autoref{alg:SF}, we have $\|\bx_t - \bx_{t-1}\| = \|\bts^{-1} \update_{t}\|$ for $t \geq 1$, leading to the following inequality:
\begin{align}
    \mathbb{E} \sum_{t=1}^{T} \|\bx_t - \bx_{t-1}\|^2 &= \bts^{-2} \mathbb{E} \sum_{t=1}^{T} \|\update_{t}\|^2 \leq 16 \mathbb{E} \sum_{t=1}^{T} \|\update_t\|^2,
\end{align}
which shows that $\cx \leq 16$.

Next, we upper bound the sum of loss decrements.
We begin with the following decomposition, using the $G$-Lipschitzness of $F$:
\begin{align}
    \sum_{t=1}^{T} \left(F(\bx_t) - F(\bw_t)\right) &= \sum_{t=1}^{T} \left(F(\bx_{t}) - F(\bw_{t-1}) + F(\bw_{t-1}) - F(\bw_t)\right) \\
    &\leq \sum_{t=1}^{T} G \|\bx_t - \bw_{t-1}\| + \sum_{t=1}^{T} \left(F(\bw_{t-1}) - F(\bw_t)\right).
\end{align}
Thus, using the closeness property \eqref{exp:diff_SF} between $\bx_t$ and $\bw_{t-1}$, it follows that
\begin{align}
    \E \sum_{t=1}^{T} \left(F(\bx_t) - F(\bw_t)\right) &\leq G \E \sum_{t=2}^{T} \etas \|\bg_{t-1}\| + \E \left[F(\bw_0) - F(\bw_T)\right] \\
    &\leq G \etas T (G + \sigma) + \subopt = 2G \Ds T \sqrt{1 - \betas} + \subopt.
\end{align}

Finally, let us substitute this upper bound, together with $\cx \leq 16$, back into the inequality in \autoref{thm:opt}. We obtain
\begin{align*}
    \E_{\tau} \left[\regnorm{\lambda}{\nabla F(\bby_{\tau})}\right] &= 3\eps + \frac{1}{\Ds T} \E \sum_{t=1}^{T} \left(F(\bx_t) - F(\bw_t)\right) \\
    &\leq 3\eps + \frac{1}{\Ds T} \left(2G \Ds T \sqrt{1 - \betas} + \subopt\right) \\
    &\leq 3\eps + 2G \frac{\eps}{7(G + \sigma)} + \frac{4 \lambda^{1/2} \eps^{-1/2} \left(1 + \frac{49(G + \sigma)^2}{\eps^2} \cdot \sqrt{\cx} \right)}{T} \cdot \subopt \\
    &\leq 4\eps + \frac{4 \subopt \lambda^{1/2} \eps^{-1/2} \left(5 \cdot \frac{49(G + \sigma)^2}{\eps^2}\right)}{T},
\end{align*}
provided that $T \geq 49 (G + \sigma)^2 \eps^{-2}$.

Thus, the desired iteration complexity bound follows.

\section{Deferred proofs of other results}

\subsection{Proof of the OMD discounted regret (\autoref{lem:omd-regret})}
\label{pf:lem:omd-regret}

To begin, let $\bv_1, \bv_2, \dots, \bv_T$ be a sequence of vectors, and consider the quadratically regularized linear loss sequence $\{\ell^{\bv}_t\}_{t=1}^T$, defined as:
\[
\forall t \in [T], \quad \ell^{\bv}_t(\cdot) = \langle \bv_t, \cdot \rangle + \frac{\mu_t}{2} \norm{\cdot}^2,
\]
where the regularization strengths $\mu_t \geq 0$ for all $t \in [T]$ are given and known to the online learner. Since the quadratic regularization term is fully known, it is standard for the online learner to incur regret only from the uncertainty in $\bv_t$.

For instance, \citet{zhang2024random} demonstrate that a version of composite objective Online Mirror Descent (OMD) \citep{beck2003mirror,duchi2010composite} provides the following regret bound for chosen decreasing step sizes $\eta_t > 0$ (i.e., $0 < \eta_{t+1} \leq \eta_t$ for all $t$):
\begin{align} \label{exp:omd-regret}
\sum_{t=1}^T \left[ \ell^{\bv}_t(\update_t) - \ell^{\bv}_t(\bu) \right]  \leq \left(\frac{2}{\eta_{T+1}} + \frac{\mu_{T+1}}{2} \right) \norm{\bu}^2 + \frac{1}{2} \sum_{t=1}^T \eta_t \norm{\bv_t}^2.
\end{align}
See their Theorem 4.1 for details.
The version of composite objective Online Mirror Descent (OMD) considered by \citet{zhang2024random} is initialized with $\update_1 = \bm{0}$ and updated as:
\begin{align} \label{exp:omd-opt}
\update_{t+1}  \coloneqq \argmin_{\update} \left[ \langle \bv_t, \update \rangle + \frac{1}{2\eta_t} \norm{\update - \update_t}^2 + \frac{\mu_{t+1}}{2} \norm{\update}^2 +  \frac{\left( \frac{1}{\eta_{t+1}} - \frac{1}{\eta_t} \right)}{2} \norm{\update}^2 \right].
\end{align}
For more details, see Section 4 therein.


Computing the argmin of the right-hand side, the update rule \eqref{exp:omd-opt} can be equivalently written as:
\begin{align}  \label{exp:omd_simp}
\update_{t+1}  \coloneqq  \frac{1}{1 + \eta_t \mu_{t+1} + \eta_t \left( \frac{1}{\eta_{t+1}} - \frac{1}{\eta_t} \right)} (\update_t - \eta_t \bv_t).
\end{align}
Now, with the choice $\bv_t \coloneqq \beta^{-t} \bg_t$, $\mu_t = \beta^{-t} \mu$, and $\eta_t = \beta^t \eta$, the update rule \eqref{exp:omd_simp} becomes: 
\begin{align}   
\update_{t+1}  \coloneqq  \frac{1}{1 + \frac{1}{\beta} \eta \mu + \frac{1}{\beta} - 1} (\update_t - \eta \bg_t),
\end{align}
which, after rearranging, precisely matches \eqref{exp:domd} in the statement of \autoref{lem:omd-regret}.

Moreover, plugging this choice into the regret bound \eqref{exp:omd-regret}, we get: 
\begin{align}
\sum_{t=1}^T \langle \beta^{-t} \bg_t, \update_t - \bu \rangle \leq \left( \frac{2}{\beta^{T+1} \eta} + \frac{\beta^{-(T+1)} \mu}{2} \right) \norm{\bu}^2 + \frac{1}{2} \sum_{t=1}^T \beta^t \eta \norm{\beta^{-t} \bg_t}^2,
\end{align}
which, after multiplying both sides by $\beta^T$, becomes:
\begin{align}  \label{exp:omd-dregret}
\dregret_T(\bu) \leq \beta^{-1} \left( \frac{2}{\eta} + \frac{\mu}{2} \right) \norm{\bu}^2 + \frac{1}{2} \sum_{t=1}^T \eta \beta^{T-t} \norm{\bg_t}^2.
\end{align}

Now let us simplify this into the discounted regret bound  in \autoref{lem:omd-regret}. Since we have $\E\norm{\bg_t}^2 \leq G^2 + \sigma^2$ for all $t$, the discounted regret bound \eqref{exp:omd-dregret} can be further upper bounded as:

\[
\E[\dregret_T(\bu)] \leq \frac{\mu}{2} \frac{\norm{\bu}^2}{\beta} + \frac{2}{\eta} \frac{\norm{\bu}^2}{\beta} + \frac{1}{2} \sum_{t=1}^{T} \eta \beta^{T-t} (G^2+\sigma^2),
\]
and
\[
\E[\dregret_T(\bu)] \leq \frac{\mu}{2} \frac{\norm{\bu}^2}{\beta} + \frac{2}{\eta} \frac{\norm{\bu}^2}{\beta} + \frac{1}{2} \frac{(G^2 + \sigma^2)}{1-\beta} \eta.
\]
Choosing $\eta = \frac{2 \norm{\bu} \sqrt{1-\beta}}{G + \sigma}$ and noticing that the right-hand side no longer depends on $T$, the discounted regret bound becomes:
\[
\forall t \in [T], \quad \E\left[\dregret_t(\bu)\right] \leq \frac{2\norm{\bu}(G + \sigma)}{\beta\sqrt{1-\beta}} + \frac{\mu}{2}\norm{\bu}^2.
\]
This gives the discounted regret bound given in \autoref{lem:omd-regret}.

 \subsection{Proof of \autoref{prop:extrapolate}}
 \label{pf:prop:extrapolate}

  From the definition of the $\bz$-iterates \eqref{def:zt}, it holds that
\begin{align*}
\bz_{t+1} - \bz_t &= ( \bx_{t+1} +\frac{1}{1-\bt} \update_{t+1}) - ( \bx_t +\frac{1}{1-\bt} \update_t) \\
&=    \left(\frac{1}{\bt} +\frac{1}{1-\bt}\right) \update_{t+1} - \frac{1}{1-\bt}\update_t  \\
&=  \frac{1}{1-\bt} \left( \left(\frac{1-\bt}{\bt}  + 1\right)\update_{t+1} - \update_t \right)  = -\frac{\eta}{1-\bt} \bg_t.
\end{align*}
Hence, the desired conclusion follows.

\subsection{Proof of \autoref{prop:parameters}}
\label{pf:prop:parameters}

From the proof of \autoref{cor:SF} (see \autoref{pf:cor:SF}), we have the relationship:
\begin{align}
    \etas \mus = \frac{\eps^2}{7(G + \sigma)^2}.
\end{align}
Since $\betas = 1 - \left(\frac{\eps}{7(G + \sigma)}\right)^2$, we can deduce the following:
\begin{align}
    1 - \bts = 1 - \frac{\betas}{1 + \etas \mus} = \frac{(1 - \betas)+ \etas \mus }{1 + \etas \mus} = \frac{\TT{\frac{\eps^2}{(G + \sigma)^2}}}{1+\TT{\frac{\eps^2}{(G + \sigma)^2}}} = \TT{\frac{\eps^2}{(G + \sigma)^2}}.
\end{align}
From this, the conclusions of \autoref{prop:parameters} follow.

\end{document}